\documentclass[journal]{IEEEtran}
\ifCLASSINFOpdf
  \usepackage[pdftex]{graphicx}
\else
\fi
\usepackage{amsmath}
\usepackage{algorithmic}
\usepackage{algorithm}
\ifCLASSOPTIONcompsoc
 \usepackage[caption=false,font=normalsize,labelfont=sf,textfont=sf]{subfig}
\else
 \usepackage[caption=false,font=footnotesize]{subfig}
\fi
\usepackage{url}

\usepackage{amsmath}
\usepackage{amssymb}
\usepackage{amsthm}
\usepackage{hyperref}
\usepackage{cleveref}
\usepackage{fancyhdr}
\usepackage{cite}
\usepackage{color}
\usepackage{tabularx}
\usepackage{threeparttable}
\usepackage{mathrsfs}
\usepackage{hyperref}
\usepackage{graphicx} 
\usepackage{xcolor}
\usepackage{svg}

\newcommand\bs[1]{\mathbf{#1}}
\newcommand\mc[1]{\mathcal{#1}}
\newcommand\bb[1]{\mathbb{#1}}
\newcommand\s[1]{\mathsf{#1}}
\newcommand\gen{\text{gen}}

\theoremstyle{plain}
\newtheorem{assumption}{Assumption}
\theoremstyle{plain}
\newtheorem{lemma}{Lemma}
\theoremstyle{plain}
\newtheorem{proposition}{Proposition}
\theoremstyle{plain}
\newtheorem{remark}{Remark}
\theoremstyle{plain}
\newtheorem{theorem}{Theorem}
\theoremstyle{plain}

\theoremstyle{plain}

\makeatletter
\newcommand{\leqnomode}{\tagsleft@true\let\veqno\@@leqno}
\makeatother

\begin{document}
%
\title{AI-in-the-Loop Sensing and Communication Joint Design for Edge Intelligence }
%
%
%

\author{Zhijie~Cai,~\IEEEmembership{Student Member,~IEEE,}
        Xiaowen~Cao,~\IEEEmembership{Member,~IEEE,}
        Xu~Chen,~\IEEEmembership{Member,~IEEE,}
        Yuanhao~Cui,~\IEEEmembership{Member,~IEEE,}
        Guangxu~Zhu,~\IEEEmembership{Member,~IEEE,}
        Kaibin~Huang,~\IEEEmembership{Fellow,~IEEE,}
        Shuguang~Cui,~\IEEEmembership{Fellow,~IEEE}
\thanks{Part of this work has been presented in 2024 IEEE 4th International Symposium on Joint Communications \& Sensing (JC\&S), Leuven, Belgium.}
\thanks{Z. Cai and G. Zhu are with Shenzhen Research Institute of Big Data and the School of Science and Engineering (SSE), The Chinese University of Hong Kong, Shenzhen, Shenzhen 518172, China.}
\thanks{X. Cao is with the College of Electronics and Information Engineering, Shenzhen University, Shenzhen 518060, China.}
\thanks{Y. Cui is with the School of Information and Communication Engineering, Beijing University of Posts and Telecommunications, Beijing 100876, China.}
\thanks{X. Chen and K. Huang are with the Department of Electrical and Electronic Engineering, The University of Hong Kong, Hong Kong.}
\thanks{S. Cui is with the School of Science and Engineering (SSE), the Shenzhen Future Network of Intelligence Institute (FNii-Shenzhen), and the Guangdong Provincial Key Laboratory of Future Networks of Intelligence, The Chinese University of Hong Kong, Shenzhen, Shenzhen 518066, China.}
}

%
%

\markboth{}%
{Shell \MakeLowercase{\textit{et al.}}: Bare Demo of IEEEtran.cls for IEEE Journals}
%



\maketitle
\begin{abstract}

Recent breakthroughs in artificial intelligence (AI), wireless communications, and sensing technologies have accelerated the evolution of edge intelligence. However, conventional systems still grapple with issues such as low communication efficiency, redundant data acquisition, and poor model generalization. To overcome these challenges, we propose an innovative framework that enhances edge intelligence through AI-in-the-loop joint sensing and communication (JSAC). This framework features an AI-driven closed-loop control architecture that jointly optimizes system resources, thereby delivering superior system-level performance. A key contribution of our work is establishing an explicit relationship between validation loss and the system's tunable parameters. This insight enables dynamic reduction of the generalization error through AI-driven closed-loop control. Specifically, for sensing control, we introduce an adaptive data collection strategy based on gradient importance sampling, allowing edge devices to autonomously decide when to terminate data acquisition and how to allocate sample weights based on real-time model feedback. For communication control, drawing inspiration from stochastic gradient Langevin dynamics (SGLD), our joint optimization of transmission power and batch size converts channel and data noise into gradient perturbations that help mitigate overfitting. Experimental evaluations demonstrate that our framework reduces communication energy consumption by up to $77 \%$ and sensing costs measured by the number of collected samples by up to $52 \%$ while significantly improving model generalization—with up to $58 \%$ reductions of the final validation loss. It validates that the proposed scheme can harvest the mutual benefit of AI and JSAC systems by incorporating the model itself into the control loop of the system.


\end{abstract}


%
\IEEEpeerreviewmaketitle

\section{Introduction}\label{sec:intro}
%
%
%
%
\IEEEPARstart{E}{xisting} prominent progress in \emph{artificial intelligence} (AI), wireless communication, and sensing have envisioned that the future networks will not only serve the sole purpose of data delivery but also actively and cooperatively collect data then learn from it, providing a large variety of tactile intelligent services and applications such as auto-pilot and metaverse \cite{liu2022integrated}, \cite{cui2021integrating}. The convergence of sensing, communications, and AI is incorporated in the 6G vision by the IMT-2030 group \cite{itu_6g}. The natural integration of these new features at the network edge agrees with the research area of edge intelligence, aiming to realize the synergy of sensing, communication, and AI for end-to-end downstream tasks \cite{liu2025integrated} \cite{chen2023view}, \cite{zhu2023pushing}, \cite{cui2023integrated}. However, the exact mechanisms of how AI models function are unclear, and models are still viewed as ``black boxes". AI-in-the-loop control, explicitly incorporating the AI model itself into the control loop of the learning process, is a possible solution to navigate the model through the learning process while the complex interaction of the model and data remains unrevealed.

Existing designs of edge intelligence are primarily based on \emph{federated learning} (FL) \cite{mcmahan2017communication}, known as \emph{federated edge learning} (FEEL). In a FEEL system, multiple \emph{Joint Sensing and Communication} (JSAC) devices, known as \emph{devices}, will collect their local data and cooperatively train a standard AI model under the coordination of a \emph{server}. 
FEEL guarantees local data privacy by allowing devices to use local models as the media of knowledge exchange rather than raw data. In particular, during training, each device updates its local model using local data and then sends it to the server so that a new global model can be created and broadcast for the next iteration. The design broadens the applicable scenarios of edge intelligence since it potentially allows more data to be put into the training set of an AI model, so the edge devices of a massive number, like mobile phones, can contribute data to the model. 
Nevertheless, there are challenges when translating the concept into reality.

The first challenge is \textbf{inefficient communication}. There is a natural demand for the services the edge intelligence system provides to return results in time. However, a FEEL system's most basic communication design uses orthogonal access, where the server rations bandwidth for heavy communication loads. The situation aggravates with the emergence of large models \cite{zhao2023survey}. Some compromising measures include lossy gradient compression \cite{tao2018esgd} and parameter-efficient fine-tuning \cite{yi2023fedlora}. Still, they are all bid to the fact that edge devices are assigned a reciprocally narrow band, meaning that the FL process will always be slowed down if more edge devices join the process. A technology named \emph{over-the-air federated edge learning} (Air-FEEL) \cite{cao2021optimized}, is a popular solution that enables devices to efficiently aggregate local updated gradient using the superposition property of wireless channels on the whole available frequency band for the task once the models are properly coded. However, devices cannot convey their local updates perfectly to the server due to the channel noise, which poses a problem for noise control and management during the learning process.

Another challenge is \textbf{inefficient data usage}. As is well known, good AI models are usually trained on a tremendously large dataset. In a FEEL system, it translates to a large data sensing and computation demand for the devices. The devices will need to continuously collect new samples for their model update, posing great difficulties since they are usually assumed to be resource-constrained like mobile phones. Moreover, although more data samples will bring more information to the model, researchers have found that some samples are more \emph{important} than others \cite{settles2009active}, implying a data efficiency gap. To improve data efficiency in edge learning systems by picking out the essential data samples, \cite{liu2020data} proposed that scheduling devices with more important samples to upload their models can accelerate model convergence, and \cite{he2020importance} proposed that selecting the data samples with a higher loss for backpropagation can reduce training latency and improve learning accuracy.
Additionally, \cite{katharopoulos2018not} proposed a gradient reweighting-based sampling method to alter the data distribution to one that guarantees maximum convergence speed \cite{richtarik2016optimal, alain2015variance}. However, although fewer samples are used in actual AI model training by applying their methods, neither lowers the sensing cost induced by data acquisition. \cite{katharopoulos2018not} even requires a larger batch to inspect the current gradient distribution. It brings no harm in a centralized learning scenario but could drain the energy in a device in a FEEL system by repetitively requiring more data, leading to worse final performance, which thus calls for a more sophisticated design on data acquisition.

One of the most difficult challenges is \textbf{poor generalizability}. The correctness and reliability of the results are fatal in an edge intelligence system. Although \emph{deep learning} is the most popular and most widely-used form of AI, it suffers from the \emph{overfitting} problem, meaning that AI models are likely to give accurate predictions for seen samples but perform drastically worse for unseen data, even if they are sampled from the same distribution where the training data is from \cite{dietterich1995overfitting}. While normal AI practices can mitigate the issue by accessing more memory \cite{ji2020history}, it is not a viable option for resource-limited edge devices. There are pioneering works investigating how to accelerate model convergence by adjusting some tunable parameters in a FEEL system, including batch sizes, sensing power\cite{liu2022toward}, 
task queueing \cite{zhang2023distributed}, and communication power control \cite{wen2023task}. However, fast convergence does not imply good inference performance. Notably, \cite{wen2023task} used a surrogate called \emph{discriminant gain} to analyze the inference loss (equivalently validation loss in this context, consistent with lowering overfitting effect), but it does not guide the design of the training phase of a FEEL system. Under the centralized setting, one of the well-known approaches to this problem is the loss regularization technique, where penalties are added to the original objective \cite{moradi2020survey}. However, loss regularization usually introduces additional computational consumption. Data augmentation is also an approach to this end by adding noise to the training data and is empirically proven useful in various AI applications \cite{shorten2019survey, feng2021survey, wen2020time}. However, the characterization of its impact on AI performance is also hardly tractable due to the AI models' highly nonlinear nature. There are compromising proposals for realization, each asserting different assumptions on the model or problem itself, making it less possible to rigorously track the relationship between AI performance and its tunable parameters \cite{kukavcka2017regularization}. Meanwhile, there is no work addressing generalizable FEEL training to the best of our knowledge except \cite{cai2024integrated}, which is a prior version of this work.

\begin{figure}
    \centering
    \includegraphics[width=0.7\linewidth]{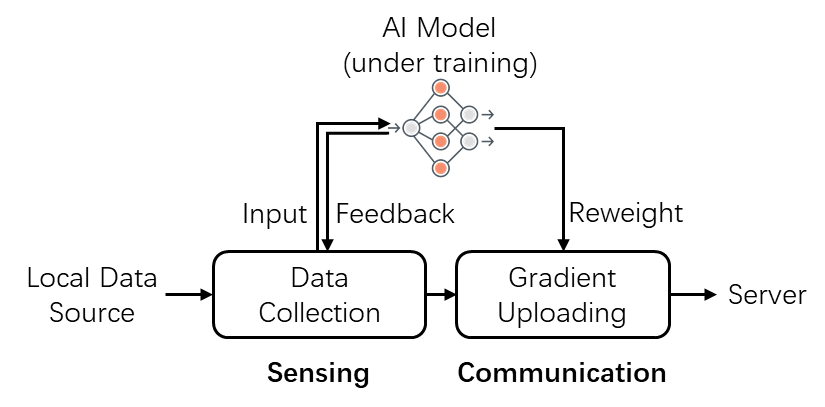}
    \vspace{-10pt}
    \caption{AI-in-the-Loop Sensing and Communication Joint Design}
    \label{fig:ai-in-the-loop}
    \vspace{-10pt}
\end{figure}

Seeing the three challenges that exist on the road to a real edge intelligence system, we try to address one key question in making FEEL a practical application in future networks: How can we guarantee low validation loss as well as fast convergence in a trained edge intelligence system? 
To this aim, a JSAC optimization on resources with smart AI-in-the-loop control under the guarantee of good generalizability\footnote{Generalizability can also be referred to the ability of a generalization model, but we use this term to describe how good a machine learning model fits a dataset consists of samples \emph{independently identical distributed} of that from the training set.} is required.

The crux of the solution is as follows. It is shown that the generalization error relates to a conceptual measurement: the \emph{mutual information} \cite{polyanskiy2014lecture} between the training data and the model weight \cite{bu2020tightening}. Intuitively, if the model weight is less related to the training data, the model generalizes better. So, a naive way to minimize this quantity will lead to a trivial solution: don't train at all. If the model is independent of the dataset, there will be no overfitting, but at the same time, it will be of no use. So we choose the validation loss instead as the objective, which can be decomposed as a summation of the \textbf{training loss} and \textbf{generalization error}, the latter part characterizes the overfitting effect. This encourages us to lower the generalization error by smartly adjusting the distribution of the training data and the model weight using tunable parameters in an edge intelligence system so that the trained model is usable but as independent of the training data as possible. We then formulate an optimization problem toward this end, which entails addressing two subproblems concerning the adjustment of the two distributions, respectively. To solve the subproblem of training data distribution adjustment, we propose a novel implementation of importance sampling that integrates the FEEL model itself into the control loop of the FEEL system. By allowing the devices to actively decide when to stop data sample collection as well as how the data samples should be weighted based on the forwarded gradients provided by the device models themselves, it builds an \textit{AI-in-the-loop} sensing control to optimize the data distribution, as a consequence, lowering the upper bound of the mutual information. To solve the other subproblem of model weight distribution adjustment, inspired by the \emph{Stochastic gradient Langevin dynamics} (SGLD) \cite{Welling2014} through Gaussian channels \cite{xiong2023fundamental}, we inject noise to the uploaded local gradient to alter the model weight distribution, where the noise level can be controlled by two tunable parameters in the system, namely, the transmission power during model update (affecting channel-induced noise), and the size of data batches (affecting data-induced noise). By integrating the two techniques as shown in \Cref{fig:ai-in-the-loop}, we obtain a powerful FEEL training scheme that enjoys both a lower validation loss and a lower cost on sensing and communication.

The findings and contributions of this work are summarized as follows.

\begin{itemize}
    \item \textbf{Explicit Validation Loss Bound:} We derive an explicit upper bound on the validation (or population) loss as a function of the tunable parameters in a JSAC system. This bound directly links key system parameters --- such as sensing batch sizes, transmission power, and gradient noise levels --- to the final model performance, thereby providing a clear guideline for system design.
    \item \textbf{Solution approach via problem decomposition:} Building on the derived bound, we formulate a problem of minimizing the validation loss bound against the tunable parameters. 
    Recognizing its inherent intractability, we decompose the problem into two decoupled subproblems. The first subproblem focuses on optimizing the data distribution by reducing the generalization error via adaptive importance sampling. The second subproblem addresses the weight (or gradient) distribution, emphasizing the control of gradient noise through effective communication resource allocation. This decomposition not only renders the optimization tractable but also lays the groundwork for the subsequent algorithmic design.
    \item \textbf{AI-in-the-loop sensing control:} 
    Motivated by the need to optimize data distribution, we introduce an innovative AI-in-the-loop sensing control protocol. In this approach, edge devices actively decide when to stop data collection by monitoring real-time gradient variance. This adaptive sensing mechanism minimizes sensing overhead while mitigating overfitting, thereby ensuring more effective data sampling.
    \item \textbf{SGLD-inspired gradient noise control:} Complementing the adaptive sensing control, we introduce an SGLD-inspired approach to regulate gradient noise by communication resource allocation optimization. Drawing inspiration from SGLD, our strategy is designed based on a solution to an optimization problem against dynamic transmission power and batch size to inject the right amount of noise in order to minimize the expected validation loss. This careful calibration not only mitigates overfitting through controlled gradient perturbations but also translates into more efficient communication --- reducing energy consumption while mitigating the overfitting effect.
    \item \textbf{Performance evaluation:} We conduct extensive simulations on two case studies under the edge intelligence paradigm. One is the handwritten digits classification task, the other is a human motion classification using wireless sensing with JSAC devices. We validate the advantages of the proposed JSAC resource allocation scheme and the AI-in-the-loop sensing control algorithm that combines into the proposed scheme. Experimental evaluations demonstrate that our framework reduces communication energy consumption by up to $77 \%$ and sensing costs by up to $52 \%$, while significantly improving model generalization—with up to $58 \%$ reductions of the final validation loss.
\end{itemize}

The remainder of this paper will be organized as follows. In \cref{sec:SMPF}, we will introduce the system model and problem formulation. In \cref{sec:PS}, we analyze the objective and split the problem. We will focus on the subproblem of reducing generalization error by importance sampling in \cref{sec:ACLGE} and the expected validation loss before importance sampling by JSAC resource allocation in \cref{sec:RAVLO}. In \cref{sec:ER}, we provide empirical results elaborated on a case study. In \cref{sec:conclusion}, we conclude this paper.

Unless further specified, symbols in the remainder of the paper that appear in Roman fonts will denote a constant or a random variable. Bold faces denote vectors. Subscript $\cdot_k$ and $\cdot_r$ means evaluating the symbol at a specific round indexed $r$ or device indexed $k$. Superscript $\cdot^{(r)}$ means the collection of evaluation of the symbol from round $1$ to $r$, equivalently $\cdot^{(r)} := \{\cdot_1, \dots, \cdot_r\}$. Subscript in sans-serif fonts denotes the task of the symbol, namely, $\cdot_\s{s}, \cdot_\s{cm}$ and $\cdot_\s{cp}$ means that the symbol is used in the sensing, communication, and computation phase, respectively. Sets will be written in calligraphic fonts, e.g., device index set $\mc{K} = \{1, \dots, K\}$. $|\cdot|$ denotes the cardinality of a set, e.g. $|\mathcal{D}_{k, r}|$ denotes the batch size of the batch $\mathcal{D}_{k, r}$, collected by device device indexed $k$ at round $r$.

\section{System Model and Problem Formulation} \label{sec:SMPF}

We consider an edge learning system consisting of an edge server and $K$ JSAC devices operating Air-FedSGD, as shown in \Cref{fig:scenario}. With the coordination of the edge server, the JSAC devices cooperate to carry out the Air-FedSGD algorithm.

\subsection{Federated Edge Learning}\label{subsec:learning_model}

\begin{figure}[t]
  \centering
  \includegraphics[scale=0.4]{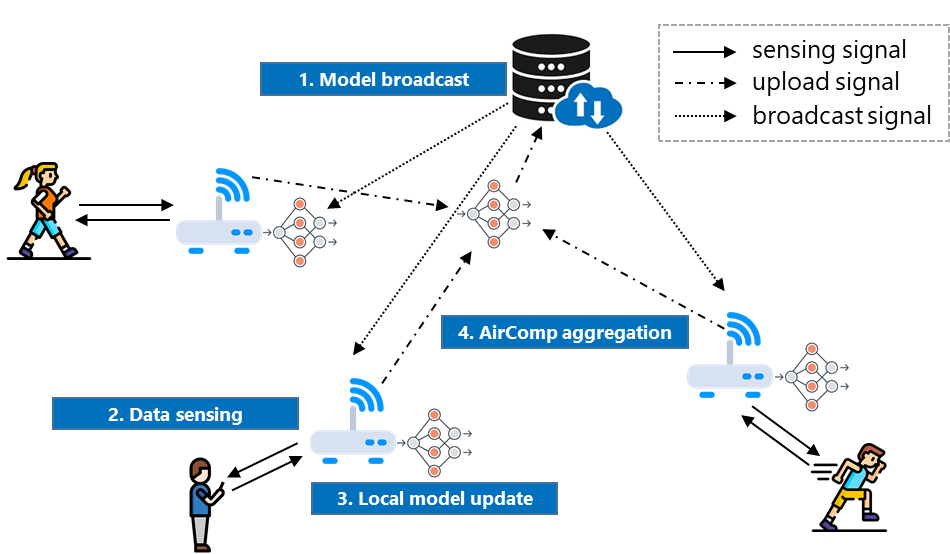}
  \vspace{-10pt}
  \caption{Federated edge learning system with Air-FedSGD empowered integrated sensing And communication.}\label{fig:scenario}
\end{figure}
We consider a classic Air-FedSGD algorithm that trains an AI model at a central server using \emph{stochastic gradient descent} (SGD) and the sensory data collected by the distributed JSAC devices \cite{murshed2021machine}. The FL algorithm executes multiple iterative training rounds until convergence, and in each round, the model weights are updated through the following three steps:
\begin{enumerate}
  \item \textbf{Model broadcasting}: In round $r$, the edge server broadcasts the latest model weights, denoted by $\bs{w}_{r-1}$, via a wireless broadcast channel to every JSAC device. 
  \item \textbf{Local gradient computation}: Upon receiving the global model, each device first collects sensory data with a sensing module and then computes a gradient vector using the fresh data and the received model. Let the data samples at the device $k$ be $\mc{D}_{k, r}$ and the corresponding gradient be $\bs{g}_{k, r} \in \mathbb{R}^{d\times 1}$ with $\bs{g}_{k, r}=\nabla_{\bs{w}} F(\bs{w}_{r-1}, \mc{D}_{k, r})$ with $F(\cdot)$ denotes the loss function. For traceability, we assume homogeneity across edge devices, namely, the devices are endowed with equal capacities of sensing, communication, and computation and share the same model architecture and size.
  \item \textbf{Global model update}: The server aggregates all local gradients into a single gradient using averaging as
  \begin{equation}\label{eq:grad_aggregate}
      \bar{\bs{g}}_r = \frac{1}{K}\sum_{k \in \mc{K}}\bs{g}_{k, r}.
  \end{equation}
  Then, the global model is updated based on an SGLD algorithm\cite{Wang2021, Welling2014}, where the new model weights are given as
  \begin{equation}\label{eq:update_rule}
    \bs{w}_r = \bs{w}_{r-1} - \eta(\bar{\bs{g}}_r+\bs{n}_r),
  \end{equation}
  where $\eta$ denotes the learning rate and $\bs{n}_r$ represents a manually added Gaussian noise with $\bs{n}_r\sim\mathcal{N}(0,\tau_r)$.
\end{enumerate}

\subsection{Sensing Model}\label{subsec:sensing_model}
The JSAC devices leverage sensing algorithms and data sampling, which acquire training datasets in real-time, to support the model training in Air-FedSGD. Specifically, at an arbitrary device, say device $k$, a {\it frequency modulated continuous wave} (FMCW) is emitted by the transmitter to illuminate the object, e.g., a human body, and the echo signal is received for sampling data $\mc{D}_{k, r}$. Consider a primitive-based method \cite{ram2010simulation} that models the scattering of sensing objects as a linear time-varying system. Then, the received echo signal is given as \begin{align}\label{eq:sensing_signal}
  \hat{\bs{x}}(t) = \sqrt{p_{\s{s}, k}}\bs{x}_1(t) + \sqrt{p_{\s{s}, k}}\bs{x}_{-1}(t) + \bs{n}_{\s{s}, k}(t),
\end{align}
where $p_{\s{s},k}$ denotes transmit power, $\bs{x}_1(t)$ denotes the temporal normalized one-hop reflective signal (i.e., transmitted from the transmitter, reflected by the human body, and finally received the receiver), $\bs{x}_{-1}(t)$ denotes the summation of all multi-hop normalized reflective signal, and $\bs{n}_{\s{s}, k}(t)$ the additive sensing noise. Here, the latency resulting from sensing signal propagation is eliminated by focusing on indoor sensing environments with short propagation distances.

Next, training datasets are generated by sampling the captured echo signals $\hat{\bs{x}}(t)$. Without loss of generality, the sampling interval is fixed as $T_0$ across devices and let $a_{k,r}$ be the number of samples for round $r$ at device $k$. The $a_{k,r}$ samples at device $k$ will be further up-sampled for forming the training data batch $\mc{D}_{k, r}$ with a predefined batch size $b_r = |\mc{D}_{k, r}|, \forall k$. As a result, the worst latency resulting from data acquisition is then given as $T_{\s{s}, r} = T_0 b_r$. Besides, both $a_{k,r}$ and $b_r$ vary over training rounds for a dynamically weighted gradient average for improving training performance, as detailed in~\cref{sec:ACLGE}.  

\subsection{Communication Model}\label{subsec:communication_model}
The edge server leverages AirComp, where each parameter in the AI model is coded into amplitudes of electromagnetic waves, the superpositioning of which yields the aggregated model to efficiently aggregate local gradients over a multiple access channel and exploits the channel-induced distortions as the desired noise for the SGLD algorithm in Equation~\ref{eq:update_rule}. We consider a single-input-single-output channel with the \emph{channel state information} (CSI) available for both edge devices and the server for performing channel compensation. Furthermore, time is slotted, and each time slot, with the duration of $T_1$, allows for the transmission of $L$ scalar. We further assume a block-fading channel that keeps unchanged over at least $t = \left\lceil \frac{d}{L} \right\rceil$ time slots for uploading a $d$-dimensional gradient vector. As a result, the communication latency for round-$r$ training is given as $T_\s{cm} = t T_1$.

In round-$r$ gradient aggregation using AirComp, each device transmits simultaneously its update $\{\bs{g}_{k, r}\}$ via the same time-frequency block, leading to the server receiving a signal
\begin{align}
  \hat{\bs{g}}_r &= \sum_{k \in \mc{K}} h_{k, r} \rho_{k,r} \bs{g}_{k, r} + \sqrt{p_n}\bs{n}_{\s{cm}}
\end{align}
where $h_{k, r}$ represents the complex channel coefficient of device $k$, $\rho_{k,r}$ denotes the transmit power control, and $\bs{n}_{\s{cm}} \sim \mc{N}(\bs{0}, \bs{I})$ denotes channel noise with power $p_n$. Following \cite{zhu2019broadband}, channel inversion is adopted to compensate for the channel fading: 
 $\rho_{k,r}=\frac{\sqrt{c_r}}{|h_{k, r}|}$, where $c_r > 0$ is the denoising factor controlling the AirComp noise. Then, $\hat{\bs{g}}_r$ is post-processed through normalization to yield the desired gradient aggregation:
 \begin{align}
  \tilde{\bs{g}}_r = \frac{1}{K \sqrt{c_r}}\hat{\bs{g}}_r
  = \bar{\bs{g}}_r + \sqrt{\frac{p_n}{ c_r}} \bs{n}_{\s{cm},r}, \label{eq:gradient_estimate}
\end{align}
and the induced gradient error is denoted by $\bs{e} = \tilde{\bs{g}}_r - \bar{\bs{g}}_r = \sqrt{\frac{p_n}{ c_r}}\bs{n}_{\s{cm}}$.
Compare this term with the noise in the update rule \eqref{eq:update_rule}, giving $c_r = p_n\eta^2 / \tau_r$.

\subsection{Performance Metric}
Recalling the goal of this paper, we try to lower the validation loss, equivalently, the population loss 
\begin{align}\label{eq:validation_loss}
  \bb{E}_{\mc{D}' \sim \mathscr{D}} F(\bs{w}, \mc{D}'),
\end{align}
rather than the commonly used empirical $F(\bs{w}, \mc{D})$, where $\mathscr{D}$ is defined as the distribution from where dataset $\mc{D}$ is sampled. This quantity tracks the expected loss of a model weight measured over a hypothetical dataset whose samples are i.i.d. with those in the training set, rather than solely on the training set. This quantity can be considered the expected validation loss in practice and takes the overfitting effect of AI models compared to the commonly used empirical loss as the optimization objective as in \cite{liu2022toward, cao2021optimized}.

\subsection{Problem Formulation}
Our objective is to minimize the validation loss in~\eqref{eq:validation_loss}, which is by nature intractable due to the lack of closed-form expressions. To address this issue, we introduce the following objective function
\begin{align}\label{eq:objective}
  J := F(\bs{w}, \mc{D}) +\gen_{\bs{q}}(\bs{w}, \mc{D}) + F^{*},
\end{align}
where $F^*$ denotes the infimum of the loss function value and $\gen_{\bs{q}}(\bs{w}, \mc{D})$ denotes the {\it generalization error} induced by sampling distribution $\boldsymbol{q}$ that is defined as \begin{equation}\label{eq:gen_error}
  \gen_{\bs{q}}(\bs{w}, \mc{D}) = \bb{E}_{\mc{D}' \sim \mathscr{D}} [|F(\bs{w}, \mc{D}') - F_{\bs{q}}(\bs{w}, \mc{D})|].
\end{equation}
It is worth noticing that in order to lower $\text{gen}_\bs{q}(\bs{w}, \mc{D})$ as much as possible, mix weights must be corrected according to the sampling weight $\bs{q}$. Otherwise, there will be induced bias between $F_{\bs{q}}(\bs{w}, \mc{D})$ and $F(\bs{w}, \mc{D}')$. It is in some works known as importance sampling \cite{alain2015variance}, and a detailed discussion will be deferred to later chapters.

Since we assume that the samples in the validation set or inference set are sampled from the same distribution from which those in the training set are sampled, this objective can be regarded as directly corresponding to validation loss or test loss. 
To exploit the collected data on resource-limited devices as much as possible, we assume that the devices passively accept the collected data samples and only actively decide when to stop collection before the resampling. So, instead of modifying $\mc{D}_r$ directly, we denote by $b_r'$ the batch size before sampling and make it one of the optimization variables. The generalization error after the resampling is denoted as $\gen_{\{\bs{q}_r, b_r'\}}(\bs{w}, \mc{D})$.

Seeing the generalization error is tractable by adopting an evaluable upper bound for SGLD, upon which minimizing population loss of a model obtained by SGLD can be achieved by minimizing $J$ under the constraints discussed below.

\subsubsection{Latency Constraint}
Following \cite{liu2022toward}, consider synchronous aggregation and ignore broadcasting latency. The latency of each round is then composed of three parts: sensing latency $T_{\s{s}, r}$, local model uploading latency $T_\s{cm}$, and local computation latency. Therein, using the common computation modeling \cite{mao2016dynamic}, the computation delay, denoted by $T_{\s{cp}, r}$, is computed by $T_{\s{cp}, r} = \frac{b_r\nu}{\phi}$, where $\nu$ denotes the CPU cycles required for executing the computation of a single sample in the local update and $\phi$ is the CPU frequency of the devices. Consequently the overall latency consumed by the round-$r$ training is given as $T_r = T_{\s{cp}, r}+T_\s{cm}+T_{\s{s}, r}$. Finally, we constrain the total latency by \begin{align}\label{c:latency}
  \sum_{r \in \mc{R}}T_r \leq T^{\max}, \tag{C1}
\end{align}
where $T^{\max}$ is a predefined total latency threshold.
\subsubsection{Energy Constraint}
Similarly, energy consumption comprises three parts: data sensing, local computation, and local model uploading, which are discussed below.
For device $k$ at round $r$, based on~\eqref{eq:sensing_signal}, the energy consumption for sensing is $E_{\s{s}, r} = T_0b_rp_{\s{s}}$. Then, according to \cite{mao2016dynamic}, the energy consumption for local computation is $E_{\s{cp}, r} = \theta\nu (\phi_{k, r})^2b_r$, where $\theta$ is effective switched capacitance that depends on the chip architecture of the device.
Next, the energy consumption for local model uploading is calculated as $E_{\s{cm}, k, r} = T_\s{cm}\rho_{k,r} = \frac{t T_1 c_{r}}{|h_{k, r}|^2}$. Finally, the system energy consumption at each round is constrained to not exceed a predefined threshold $E^{\max}$:
\begin{align}
  \max\limits_{k \in \mc{K}} \sum_{r \in \mc{R}} \left(E_{\s{s}, r} + E_{\s{cp}, r} + E_{\s{cm}, k, r}\right) \leq E^{\max}. \tag{C2}
\label{c:energy}\end{align}

\subsubsection{Peak Power Constraints}
Finally, the peak power constraints are adopted for communication and sensing:
\begin{align}
  & 0 < c_r \leq P_{\s{cm}}^{\max}|h_{k, r}|^2, \quad \forall k \in \mc{K}, \tag{C3} \label{eq:C3} \\
  & 0 < P_{\s{s}}^{\min} \leq p_{\s{s}} \leq P_{\s{s}}^{\max}, \quad \forall k \in \mc{K}. \tag{C4}\label{c:power}
\end{align}

Under the three types of constraints specified above, the optimization problem of population loss can be formulated as 
\begin{equation}\tag{P1}\label{P1} \leqnomode
\begin{aligned}
  \min\limits_{\{c_r, b_r, b_r', \bs{q}_r\}} &\quad J \\
  \text{s.t.}&\quad\text{(C1)-(C4)}. 
\end{aligned} 
\end{equation}

Problem~\ref{P1} is a dynamic programming (DP) problem that optimizes tunable parameters, including channel inversion coefficient $c_r$, pre-sampled batch size $b_r'$, post-sampled batch size $b_r$, data batch $\mc{D}_r$, and data resampling weight $\bs{q}_r$ for all training rounds within temporally dynamic channel conditions is complicated. To overcome this issue, we perform theoretical analysis on the validation loss, demonstrating the roundwise decomposition with proper bounding of the validation loss and further transforming Problem~\ref{P1} into tractable time-independent sub-problems with decoupled variables, as detailed in the next section.

\section{Objective Analysis and Problem Decomposition}\label{sec:PS}
The key idea behind the objective analysis is to replace the validation loss with a summation of roundwise error reductions. Specifically, the objective of Problem~\ref{P1} can be rewritten as
\begin{align}\label{P2}
   J &:= \sum_{r \in \mc{R}}\left(\Delta F_r + \Delta G_r\right),
\end{align}
where $\Delta F_r := F(\bs{w}_r, \mc{D}^{(R)}) - F(\bs{w}_{r-1}, \mc{D}^{(R)})$ and $\Delta G_r := \gen_{\{\bs{q}_r, b_r'\}}(\bs{w}_r, \mc{D}^{(R)}) - \gen_{\{\bs{q}_r, b_r'\}}(\bs{w}_{r-1}, \mc{D}^{(R)})$, representing the increment of empirical loss reduction and generalization error at round $r$, respectively. 
\begin{equation}\tag{P2} \label{P2} \leqnomode
\begin{aligned}
  \min\limits_{\{c_r, b_r, \mc{D}_r, \bs{q}_r\}} &\quad \sum_{r \in \mc{R}} \left( \Delta F_r + \Delta G_r \right) \\
  \text{s.t.}&\quad\text{(C1)-(C4)}. 
\end{aligned}
\end{equation}

The analysis is built upon the assumptions commonly adopted in relevant works~\cite{bottou2018optimization, cao2021optimized, liu2022toward}. 

\begin{assumption}[Lipschitz-continuous objective gradients]\label{ass:l_continuous}
  There exists a constant $0 < L < \infty$ such that \begin{align}
    \|\nabla F(\bs{w}) - \nabla F(\bar{\bs{w}})\| \leq L\|\bs{w} - \bar{\bs{w}}\|. \notag
  \end{align}
\end{assumption}

\begin{assumption}[Polyak-Łojaciewicz inequality]\label{ass:plineq}
  There is an optimal loss function value of $F$, denoted by $F^*$, and the optimality gap $\gamma = F(\bs{w}) - F^*$ satisfies \begin{align}
  \|\nabla F(\bs{w})\|_2^2 \geq 2 \delta \gamma, \notag
\end{align}
where $\delta \geq 0$ is a constant.
\end{assumption}

\begin{assumption}[First moment limits of sample gradients]\label{ass:first_moment_lim}
  There exists constants $\mu_F, \mu_G > 0$ such that \begin{align}
    \nabla F(\bs{w}_r)^T\bb{E}[\bs{g}_r] &\geq \mu_F \|\nabla F(\bs{w})\|_2^2, \label{eq:muF}\\
    \nabla F(\bs{w}_r)^T\bb{E}[\bs{g}_r] &\geq \mu_G \bb{E}[\|\bs{g}_r\|^2], \label{eq:muG}
  \end{align}
  where $\bs{g}_r$ is the roundwise gradient estimation, the mean of the samplewise gradient estimation at round $r$.
\end{assumption}

\begin{assumption}[Second moment limits of sample gradients]\label{ass:second_moment_lim}
  There exist constants $M, M_E, M_V > 0$ such that for the samplewise gradient estimation $\bs{g}(\bs{w}_{k,r}, \bs{d})$ at model $\bs{w}_{k,r}$ and data batch $\bs{d}$, it holds that \begin{align}
    \bb{V}[\bs{g}(\bs{w}_{k, r}, \bs{d})] \leq& M_v + M_V \|\nabla{F}(\bs{w}_{r-1})\|_2^2, \notag \\
    \bb{E}[\bs{g}(\bs{w}_{k, r}, \bs{d})]^2 \leq& M_e + M_E \|\nabla{F}(\bs{w}_{r-1})\|_2^2, \notag \\
    \bb{E}[\|\bs{g}(\bs{w}_{k, r}, \bs{d})\|_2^2] \leq& M_e + M_v + (M_E + M_V) \|\nabla{F}(\bs{w}_{r-1})\|_2^2. \notag 
  \end{align}
\end{assumption}

Importantly, \Cref{ass:first_moment_lim} states that in expectation, $\bs{g}_r$ heads close to the direction of $\nabla{F}$. If $\bs{g}_r$ is an unbiased statistic of $\nabla{F}$, the inequality will turn to equality with $\mu_F = \mu_G = 1$. \Cref{ass:second_moment_lim} states that the sample gradients are mildly restricted in terms of norm and are allowed to grow quadratically to the norm of $\nabla F(\bs{w}_{r-1})$.
  Moreover, averaging $\bs{g}(\bs{w}_{k, r}, \bs{d})$ will lead to a drop of the upper bounds, for example, 
  \begin{align}
    \bb{V}[\bs{g}_r] &\leq \frac{M_v}{Kb_r} + \frac{M_V}{Kb_r} \|\nabla F(\bs{w}_{r-1})\|_2^2, \notag \\
    \bb{E}[\|\bs{g}_r\|_2^2] &\leq M_e + \frac{M_v}{Kb_r}+ \left(M_E + \frac{M_V}{Kb_r}\right) \|\nabla F(\bs{w}_{r-1})\|_2^2. \notag
  \end{align}

Therefore, it is verified that the empirical loss is upper bounded by the gradient norm and some tunable parameters.

\begin{proposition}[Upper bounding loss descent by gradient norms]\label{prop:norm_grad}
  Given \Cref{ass:l_continuous}, \ref{ass:first_moment_lim} and \ref{ass:second_moment_lim}, taking  $0 < \eta \leq \frac{\mu_F}{L\left(M_E + \frac{M_V}{Kb_r}\right)}$, it establishes
  \begin{align}\label{eq:bounded_loss}
    &\bb{E}[F(\bs{w}_r)] - F(\bs{w}_{r-1}) \notag \\ 
    \leq&-\frac{\eta \mu_F}{2} \|\nabla F(\bs{w}_{r-1})\|_2^2 + \frac{L \tau_r \eta^2}{2K} + \frac{L \eta^2}{2} \left(M_e + \frac{M_v}{Kb_r}\right) \notag \\
    =&\Delta \bar{F}_r.
  \end{align}
\end{proposition}
\begin{proof}
    See \Cref{app:norm_grad}.
\end{proof}
According to \Cref{prop:norm_grad}, the upper bound of empirical loss reduction of each training round is irrelevant from both batch population $\mc{D}_r$ and the importance associated with each sample $\bs{q}_r$. 
{\color{black}This is because the importance sampling technique is unbiased, so choices of $\bs{q}$ made no impact on $\bb{E}[\bs{g}_r]$, so the empirical bounds rely solely on the unbiasedness assumption. In contrast, the batch population and sample importance can be altered with the unbiasedness assumption preserved.} This suggests that Problem~\ref{P2} can written as \begin{equation}
\begin{aligned}
    \min\limits_{\{b_r, c_r\}}\sum_{r \in \mc{R}}& \left( \Delta F_r + \Delta \hat{G}_r \right) + \min\limits_{\{b_r', \bs{q}_r\}}\sum_{r \in \mc{R}} \left(\Delta G_r - \Delta \hat{G}_r \right) \\
    \text{s.t.} &\quad \text{(\ref{c:latency})-(\ref{c:power})},
\end{aligned}\notag\end{equation}
where $\Delta \hat{G_r}$ denotes the roundwise generalization error increment before importance sampling, and the second summation refers to the difference of this quantity brought by the importance sampling in the whole training phase.

Since $b_r'$ and $\bs{q}_r$ do not enter the constraints directly, we split the problem into two subproblems and separately optimize them, with the first subproblem 
focused on reducing the generalization error through data distribution optimization by importance sampling:
\begin{equation}\leqnomode\tag{P3'}\label{P3'}
\begin{aligned}
    \min\limits_{\{b_r', \bs{q}_r\}}\sum_{r \in \mc{R}} \left(\Delta G_r - \Delta \hat{G}_r \right),
\end{aligned}\notag\end{equation}
and the second centered on minimizing the validation loss through weight distribution optimization by (gradient) noise control:
\begin{equation}\leqnomode\tag{P4}\label{P4}
\begin{aligned}
    \min\limits_{\{b_r, c_r\}}\sum_{r \in \mc{R}}& \left( \Delta F_r + \Delta \hat{G}_r \right) \\
    \text{s.t.} &\quad \text{(\ref{c:latency})-(\ref{c:power})}. 
\end{aligned}\notag\end{equation}

To assure the unbiasedness and efficiency of variance reduction, we use a version of importance sampling as in \cite{alain2015variance}, which focuses only on the current data batch. So we target an approximation of Problem~\ref{P3'} as follows, \begin{equation}\leqnomode\tag{P3}\label{P3}
\begin{aligned}
    \sum_{r \in \mc{R}} \min\limits_{\{b_r', \bs{q}_r\}}\left(\Delta G_r - \Delta \hat{G}_r \right).
\end{aligned}\end{equation}
It deserves mentioning that we relax the causality of iterative training in Problem~\ref{P3'}. Hence, the solution is sub-optimal. However, it still shows superiority against the benchmarks.
{\color{black} This is an often employed technique in research fields like black gradient descent \cite{beck2013convergence, bradley2011parallel}.} 

Therefore, a solution for~\ref{P2} can be obtained by solving Problem~\ref{P3} and~\ref{P4}, detailed in Section~\ref{sec:ACLGE} and Section~\ref{sec:RAVLO}, respectively.


\section{AI-in-the-loop Sensing Control for Generalization Error Reduction} \label{sec:ACLGE}


In this section, we target Problem~\ref{P3}. We first demonstrate through theoretical analysis that generalization error relates to the gradient variance during the model training phase while the latter is controllable through \emph{Importance Sampling}. Then, a simple yet
effective AI-in-the-loop algorithm to optimize the increment brought by the importance sampling of roundwise generalization error increment $\Delta G_r - \Delta \hat{G}_r$ is proposed based on this observation. 

\subsection{Generalization Error Analysis}

The generalization error is analyzed under a mild and reasonable assumption similar to \cite{Wang2021} that considers the loss function to reveal the following sub-Gaussian property.

\begin{assumption}[Sub-Gaussian objective]\label{ass:subgaussian}
  Samplewise loss function $f(\bs{w}, \bs{d})$ at any model $\bs{w}$ measured on any data sample $\bs{d}$ is $\sigma$-sub-Gaussian, which mean that for all $t \in \bb{R}$, \begin{align}
    \log \bb{E}[\exp(t(f(\bs{w}, \bs{d}) - \bb{E}[f(\bs{w}, \bs{d})]))] \leq \frac{\sigma^2 t^2}{2}. \notag
  \end{align}
\end{assumption}

In particular, if a function $f$ is bounded between two constants $a$ and $b$, this assumption is naturally satisfied by setting $\sigma = (b - a) / 2$. Then, given $f$ is $\sigma$-sub-Gaussian, we obtain the following conclusion.

\begin{lemma}\label{lm:gen_error}
  Denote by $W_r, D_r$ the random variable of the model and the data batch at round $r$, respectively, with a joint probability $P_{W_{r-1}, D_r}$. Then the bounded generalization error under \Cref{ass:subgaussian} is given by:
  \begin{align}
    \Delta G_r \leq \sqrt{\frac{2 \sigma^2}{B} I(W_r; D_r)} \leq \frac{\sigma \eta}{K B} \sum_{r'=1}^r \sqrt{\frac{\bb{V}[\bs{g}_{k, r', i}]}{K \tau_{r'} b_{r'}}}, \label{eq:gen_error}
  \end{align}
  where $B = \sum_{r=1}^R b_{k, r}$ and $\bb{V}[\bs{g}_{k, r', i}] = \bb{E}[\|\bs{g}_{k, r', i} - \bb{E}[\bs{g}_{k, r', i}]\|_2^2]$ is the gradient variance, with the expectations taken over the randomness $(\bs{w}_{r-1}, \mc{D}_{k, r}) \sim P_{W_{r-1}, D_{k, r}}$.
\end{lemma}
\begin{proof}
    See \Cref{app:var_grad}. More details can be found in Corollary 1 from \cite{Wang2021}.
\end{proof}

{\color{black}
\begin{remark}
\label{rmk:ge}
From Lemma \ref{lm:gen_error}, it is observed that if the model is more informative about the used data, the generalization error will be more significant. Moreover, the informativeness can be bounded by a term related to the variance of the gradient estimates. In other words, if $I(W_r; D_r)$ is small, indicating that the model is indifferent to the varying data at a certain model, it implies that the model is less biased by the training data and possesses better generalizability, as empirically verified in \cite{jiang2019fantastic}.
\end{remark}
}

\subsection{Variance Reduction by Importance Sampling}

To reduce the gradient variance, and hence, $\Delta G_r$, we exploit an importance sampling method as discussed in the sequel. As discussed in the {\color{black} Remark }{\color{black}\ref{rmk:ge}}, in each training round, it is intended to compute the exact expectation of gradient $\bb{E}[\bs{g}]$ for model updating so as to reduce the generalization error. This process is hindered by prohibitive computational overhead, leading to the stochastic approximation of gradients: 
\begin{align}
    \bb{E}[\bs{g}] \approx \bb{E}_{p_i}[\bs{g}_i], \notag
\end{align}
where we only need to compute the approximation $\bb{E}_{p_i}[\bs{g}_i]$ by first sampling some finite samples $\bs{g}_i$ according to probability distribution $p_i$ and average it. 

To improve the quality of the estimation, a viable approach is to reweight the samples by another distribution $q_i$. As is shown in \eqref{eq:gen_error}, after the reweighting, an adjusting term $p_i / q_i$ should be attached to $\bs{g}_i$ to ensure this is an unbiased estimation to lower generalization error as much as possible, leading to 
\begin{equation}\label{eq:optimal_sampling}
    \bb{E}[\bs{g}] \approx\bb{E}_{q_i}\left[\frac{p_i}{q_i}\bs{g}_i\right].
\end{equation}
Note that improving the performance of the final model by only adjusting the distribution is possible in existing works\cite{zhao2015stochastic, needell2014stochastic}. An effective choice of the distribution to reduce the variance of the estimation is setting $q_i \propto \|\bs{g}_i\|_2$ \cite{alain2015variance}, as presented in the following lemma.

\begin{lemma}[Maximal variance reduction\cite{alain2015variance}]\label{lm:varred}
Assume the original data distribution is uniform, specifically, $p_i = u_i = 1 / b$. Denote the adjusted sample variance operator by $\bb{S}$, and then we have \begin{align}
    \bb{S}_{u}^b(x_i) = \frac{1}{b - 1} \sum_{i=1}^{b - 1} (x_i - \bb{E}_{u}^b[x_i]) ^ 2. \notag
\end{align}We can achieve maximal variance reduction $v$ as a function of the number of actual batch size $b$ with \begin{align}
    v(b) =& \bb{V}_u^b[\{\bs{g}_{k, r, i}\}_{i=1}^b] - \bb{V}_q^b[\{\bs{g}_{k, r, i}\}_{i=1}^b] \notag \\ 
    \leq & \left(\bb{E}_u[\|\bs{g}_{k, r, i}\|_2^2] - (\bb{E}_u[\|\bs{g}_{k, r, i}\|_2\right) \notag \\
    = & \bb{S}_u^b[\|\bs{g}_{k, r, i}\|_2], \notag
\end{align}
and the superscript on the operator $\bb{E}, \bb{V}$ or $\bb{S}$ denotes the quantity that is computed on $b$ samples sampled against the original probability $u$. 
\end{lemma}

\begin{proof}
    See Supplementary in \cite{katharopoulos2018not}.
\end{proof}

Therefore, consider a fixed value of $\bb{V}_u[\{\bs{g}_{k, r, i}\}_{i=1}^b]$, one can minimize the generalization error by maximizing $\bb{S}_u^b[\|\bs{g}_{k, r, i}\|_2]$ with proper importance assignments in~\eqref{eq:optimal_sampling}.

Another interpretation for the gain of importance sampling is provided additionally. Consider a general expression for AI model training as an iterative process \begin{align}
    P_{W_r} = P_{D_{r}} P_{W_r | W_{r-1}, D_{r}}, \notag
\end{align}
where $D_r$ and $W_r$ are random variables for the newly acquired data batch to be fed into the AI model\footnote{We consider the algorithms that model only updates on newly acquired data batch, which is a general practice.}, and the model at $r$-th step, respectively. The training algorithm is then a conditional probability that outputs $W_r$ given previous round model $W_{r-1}$ and data batches $D_r$ at this round. From the expression, it is observed that the final model depends on two factors (distributions), the data sampling $P_{D_r}$ and the training algorithm $P_{W_r | W_{r-1}, D_r}$. Hence, apart from finding a ``better" algorithm, another way to improve the performance of the final AI model is to optimize data sampling distribution. In other words, ``important" samples should be endowed with a higher probability of being fed into the model. 

\subsection{AI-in-the-loop Sensing Control Protocol for Importance Sampling}

\Cref{lm:varred} above indicates that the variance reduction by importance sampling equivalently maximizes the sample variance of gradient norms from the original batch sampled against uniform sampling. However, in previous works like \cite{katharopoulos2018not}, using this property requires a larger batch first to inspect the gradient norm distribution, then downsample to a smaller batch. This process will cause additional sensing overhead in our setting since part of the computed gradients will not be used in gradient descent due to downsampling.

To avoid excessive sensing cost, upsampling is adopted to achieve importance sampling while ensuring equal sizes of the final batch. Specifically, we will start from a smaller batch with size $\underline{b}$, gradually adding newly collected samples to the batch, and observe the expected $v(\bar{b})$ after we upsample the current batch using the optimal sampling rule \eqref{eq:optimal_sampling} to a fixed batch size $\bar{b}$, and until the variance reduction reaches satisfaction. With regards to \Cref{lm:varred}, it is observed that the quantity of variance reduction $v_b$ is connected to the number of samples in a batch. It motivates us to investigate the pattern of how the variance reduction $v(\bar{b})$ varies against $b$. 

\begin{proposition}[Mean and variance of possible variance reduction]\label{prop:var_svar}
    Taking $\{\|g\|_i\}$ as realizations of the same random variable, with its $k$-th central momentum $\mu_{k}$, when $\mu_{k}$ exists and are finite at $k = 1, 2, 4$, we have \begin{align}
        \mu_{v_{\bar{b}}, 1} = \frac{b-1}{\bar{b}} \mu_{2}, 
        \quad \mu_{v_{\bar{b}}, 2} = \frac{\mu_{4}}{\bar{b}} - \frac{(b-3)\mu_{2}^2}{(b-1)\bar{b}} , \notag 
    \end{align}
    where $\mu_{v_{\bar{b}}, 1}$ and $\mu_{v_{\bar{b}}, 2}$ denotes the first and second momentum, namely, the mean and variance of the variance reduction after upsampling to batch size $\bar{b}$.
\end{proposition}

\begin{proof}
    The proof includes two steps. The first step is to observe that $\bb{V}[\{\bs{g}_{k, r, i}\}_{i=1}^b]$ will multiply by $b/\bar{b}$ times after the upsampling. The second step is simply tracking the mean and variance of the sample variance, which has well-established results in statistics textbooks. We refer the readers to Results 3 and 7 in \cite{o2014some} for complete proof.
\end{proof}


\begin{remark}
    From the proposition above, it is observed that as the number of samples $b$ increases, the expected sample variance increases, while the variance of sample variance decreases. Although computing the specific values of $\mu_2$ and $\mu_4$ is prohibitively expensive, it reveals that there could be some ``lucky draws" when slowly adding new samples into the batch, while a small batch can provide a large variance reduction in the upsampled batch.
\end{remark}



\Cref{prop:var_svar} implies that the best choices of actual batchsize $b$ may be varying during the training. This is because model parameters $\bs{w}_r$ are different at each round, inducing different distributions of $\{\|g\|_i\}$ with corresponding values of $\mu_2$ and $\mu_4$. However, computing $\mu_2$ and $\mu_4$ explicitly is again prohibitively expensive, so an efficient way of finding the best batch size is using AI-in-the-loop control. Specifically, it lets the model decide when the best time to stop collecting new data is so the algorithm involves the model continually inspecting the output intermediates and adjusting its behaviors. To facilitate this, an adaptive threshold is set to record the computed variance of gradients. When a round starts, the model collects a small batch, calculates the expected variance reduction, and compares it with an adaptive threshold $\theta_r$ that updates using an exponential moving average. If above the threshold, the batch is considered ``lucky," and no more new samples will be added to this batch. Then, the batch will be upsampled to size $\bar{b}$ and ready for backpropagation, while the threshold will be lifted a little based on how far $v_{\bar{b}}$ overshot the threshold. Otherwise, the model will try to acquire a new data sample, add it to the batch, and repeat the comparison. If the model collects $\bar{b}$ samples and still does not find a ``lucky draw", the batch will be ready for backpropagation, and the threshold will be accordingly lowered, as the model itself is in the control loop. So, we name it AI-in-the-loop sensing control. 

Combining the sensing control algorithm and the data collection strategy, we present a simple but proven effective algorithm as summarized in \Cref{alg:boot}. It not only reduces the variance of gradient norms as much as possible by altering the data distribution to the optimal data distribution required as in \Cref{lm:varred} but also releases the sensing workload by collecting as few samples as possible.

\begin{algorithm}
\caption{Sensing control with sample gradient reweighting}
\label{alg:boot}
    \begin{algorithmic}
    \REQUIRE Minimum batch size $\underline{b} \geq 1$, maximum batch size $\bar{b} \geq 1$, exponential smooth factor $0 \leq \alpha \leq 1$, previous sample variance threshold $\bar{\theta} \geq 0$.
    \ENSURE Gradient batch $\bs{g}_r = (\bs{g}_1, \dots, \bs{g}_{\bar{b}})$, sample variance threshold $\bar{\theta}$. 
    \STATE Uniformly sample a small batch $\mc{D} \leftarrow [\bs{d}_{1}, \dots, \bs{d}_{\underline{b}}]$. 
    \STATE Calculate the gradient batch $\bs{g}^b \leftarrow [\bs{g}_1, \dots, \bs{g}_{\underline{b}}]$.
    \STATE Calculate the sample gradient norms $n_g \leftarrow [g_1, \dots, g_{\underline{b}}]$. 
    \STATE $\theta \leftarrow S^{n=\underline{b}} (\bs{n}_g)$. 
    \STATE Current batch size $b \leftarrow \underline{b}$.
    \WHILE{$b < \bar{b}$ \OR $b \theta / \bar{b} < \bar{\theta}$}
    \STATE Uniformly sample one sample $\bs{d}_{b+1}$ and calculate its gradient $\bs{g}_{b+1}$.
    \STATE Compute sample gradient norm $g_{b+1}$.
    \STATE $\bs{g}_r \leftarrow \text{concat}(\bs{g}^b, [\bs{g}_{b+1}])$. 
    \STATE $\bs{n}_g \leftarrow \text{concat}(\bs{n}_g, [g_{b+1}])$
    \STATE $b \leftarrow b + 1$.
    \STATE $\theta \leftarrow \bb{S}^b (n_g)$. 
    \ENDWHILE
    \STATE Compute sampling probability $q = n_g / \sum {n_g}$.
    \STATE Resample $\bs{g}^b$ according to $q$ to batch size of $\bar{b}$.
    \STATE $\bar{\theta} \leftarrow \alpha \theta + (1 - \alpha) \bar{\theta}$.
    \RETURN $\bs{g}_r, \bar{\theta}$
    \end{algorithmic}
\end{algorithm}

\begin{remark}
    By setting $\bar{b}$ as the desired batch size $b_r$ required in \cref{P2}, it guarantees that the algorithm will always decrease the number of actual data samples (before upsampling) needed for the training. Moreover, the algorithm introduces only a tiny additional computational overhead for repeatedly calculating the sample variance of the gradient norms, while none of the forwarded gradients will be discarded.
\end{remark}


\section{SGLD-inspired Gradient Noise Control for Validation Loss Optimization}\label{sec:RAVLO}

In the preceding section, with Problem~\ref{P3} optimizing $\{b_r', \bs{q}_r\}$ addressed by an AI-in-the-loop sensing control protocol, we target Problem~\ref{P4}: to minimize the validation loss with a data distribution given. To this end, we will bring the empirical loss and the generalization error under the same bound and make attempts to lower this bound by allocating $\{c_r, b_r\}$ over rounds.

\subsection{Problem Transformation}
First of all, a unified upper bound of the validation loss is derived as follows. Based on~\Cref{prop:norm_grad}, given \Cref{ass:l_continuous}, \ref{ass:first_moment_lim} and \ref{ass:second_moment_lim}, taking a sufficiently small $\eta$ establishes
  \begin{align}
    \bb{E}[F(\bs{w}_r) - F(\bs{w}_{r-1})] \leq -\frac{\eta \mu_G}{2} \bb{V}[\bs{g}_r] + \frac{L \tau_r \eta^2}{2K}.\label{eq:var_grad}
  \end{align}
which offers an upper bound of empirical loss reduction in terms of gradient variance.
Next, combining \Cref{lm:gen_error} and~\eqref{eq:var_grad} yields a tractable bound of $J_r$ in \ref{P2} as follows.


\begin{theorem}[Upper bound of one-round objective]\label{thm:J_r}
  $J_r$ is upper bounded by
    \begin{align}\label{eq:expandJ}
        J_r \leq & -\delta \eta \mu_F \gamma_{r-1} + \frac{L p_n \eta^2}{2K c_r} + \frac{LM_v \eta^2}{2 K b_r}  + \frac{\gamma_{r-1} \sigma c_r}{B K p_n} + C \overset{\triangle}{=} \bar{J}_r, 
    \end{align}
    with
    \begin{align}
        C = \frac{L M_e \eta^2}{2} + \frac{\sigma \eta}{2 B K^2 \mu_G} + \frac{L \sigma \eta^2}{2 B K^2}. \notag
    \end{align}
    given all the contraints introduced by \Cref{prop:norm_grad}, \eqref{eq:var_grad} and all assumptions satisfied.
\end{theorem}

\begin{proof}
    See \Cref{app:j_r}.
\end{proof}
Based on \Cref{thm:J_r}, the objective function of \ref{P1} can be approximated by an explicit function $\bar{J}_r$. Furthermore, the set of constraints can be further simplified using the following conclusion.

\begin{lemma}[Extended from \cite{liu2022toward}]
    Sensing power exceeding a threshold can guarantee the quality of the data. Then it can be verified that constraints \Cref{c:latency} - {\rm (\ref{c:power})} reduce to \begin{align}\label{eq:reduced}
      \sum_{r \in \mc{R}} b_r &\leq \min\left\{\frac{T^{\max} - t R T_1}{T_0 + \frac{\nu}{\phi}}, \frac{E^{\max} - \sum_r t T_1 \frac{c_r}{|\underline{h}|^2}}{P_{\s{s}}^{\min}T_0 + \theta \nu \phi^2}\right\} \notag \\
      &\overset{\triangle}{=}Q(\{c_r\}). 
    \end{align}
    It is observed that increasing $b_r$ always leads to a $\bar{J}_r$. Since the energy for sensing and computation are generally linear to $\sum_r b_r$, reserving more of them by reducing energy for communication is always beneficial. Hence, we will try to find a scheme that yields good results and has a low communication energy consumption. 
\end{lemma}

Using the above results, a tractable reformulation of \ref{P3} is
\begin{equation}\leqnomode\tag{P4}\label{P4}
\begin{aligned}
    \min\limits_{\{c_r, b_r\}} &\quad \sum_{r \in \mc{R}} \bar{J}_r(c_r, b_r)\\
    \text{s.t.} &\quad \sum_{r \in \mc{R}} b_r \leq Q(\{c_r\}).
\end{aligned}
\end{equation}
\subsection{Optimal Communication Resource Allocation}
The new problem~\ref{P4} reveals the convexity property, allowing for an optimal closed-form solution as follows.

First, it is observed that for any $b_r$ sequence satisfying $\sum_{r \in \mc{R}} b_r \leq Q(\{c_r\})$, the lowest value of $\bar{J}_r$ is achieved by setting \begin{align}
    c_r^{\star} = \sqrt{\frac{B L p_n^2 \eta^2}{2 \sigma \gamma_{r-1}}}. \notag
  \end{align}
This can be simply proved by setting $\partial \bar{J}_r / \partial c_r = 0$ with verifying that $\partial^2 \bar{J}_r / \partial c_r^2$ monotonically increases and being negative at some $c_r > 0$. 
{\color{black}\begin{remark}
  From the optimal power control $c_r^{\star}$, it gives
  \begin{align}
    \frac{c_r}{c_0} = \sqrt{\frac{\gamma_0}{\gamma_{r-1}}}, \notag
  \end{align}
  that suggests {\color{black}a more noisy gradient should be used at the beginning phase of the training to avoid high bias towards the training data, leaving room for communication overhead reduction in the AirComp system. The gradient should be gradually more precise in order to reach convergence}. 
\end{remark}}

Then, since $c_r^{\star} = \sqrt{\frac{B L p_n^2 \eta^2}{2 \sigma \gamma_{r-1}}}$ is independent of $b_r$ and $b_r$ only enters into the third term of $J_r$ in~\eqref{eq:expandJ}, the problem of solving $\{b_r\}$ can be equivalently re-expressed from~\ref{P4} and~\ref{P1} as
\begin{align}
    \min\limits_{\{b_r\}} &\quad \sum_{r \in \mc{R}}\frac{LM_v \eta^2}{2 K b_r}\\
    \text{s.t.} &\quad \sum_{r \in \mc{R}} b_r \leq Q(\{c_r\}). \notag
\end{align}

Obtaining the optimal solution of $b_r$ for round $r$ is impossible since it requires the parameters involved in future rounds. 
{\color{black}However, we can view this problem as redistributing $Q(\{c_r\})$ into different rounds to minimize a reciprocal $b$ term, a good tactic should be letting $b_r = b$,}
where $q$ is a parameter combining all of the system-specific parameters, e.g., Lipschitz constant $L$, first moment limits parameter $\mu_F$, learning rate $\eta$, etc. The parameter $q$ governs how inaccurate the gradients can be (allowing a lower communication overhead) that disturbs the aggregated gradients, thus making the scheme general to different problems theoretically.

\section{Simulations} \label{sec:ER}

\begin{figure*}[t]
\centering
\subfloat[Versus the number of elapsed steps.]{
		\includegraphics[width=0.48\textwidth]{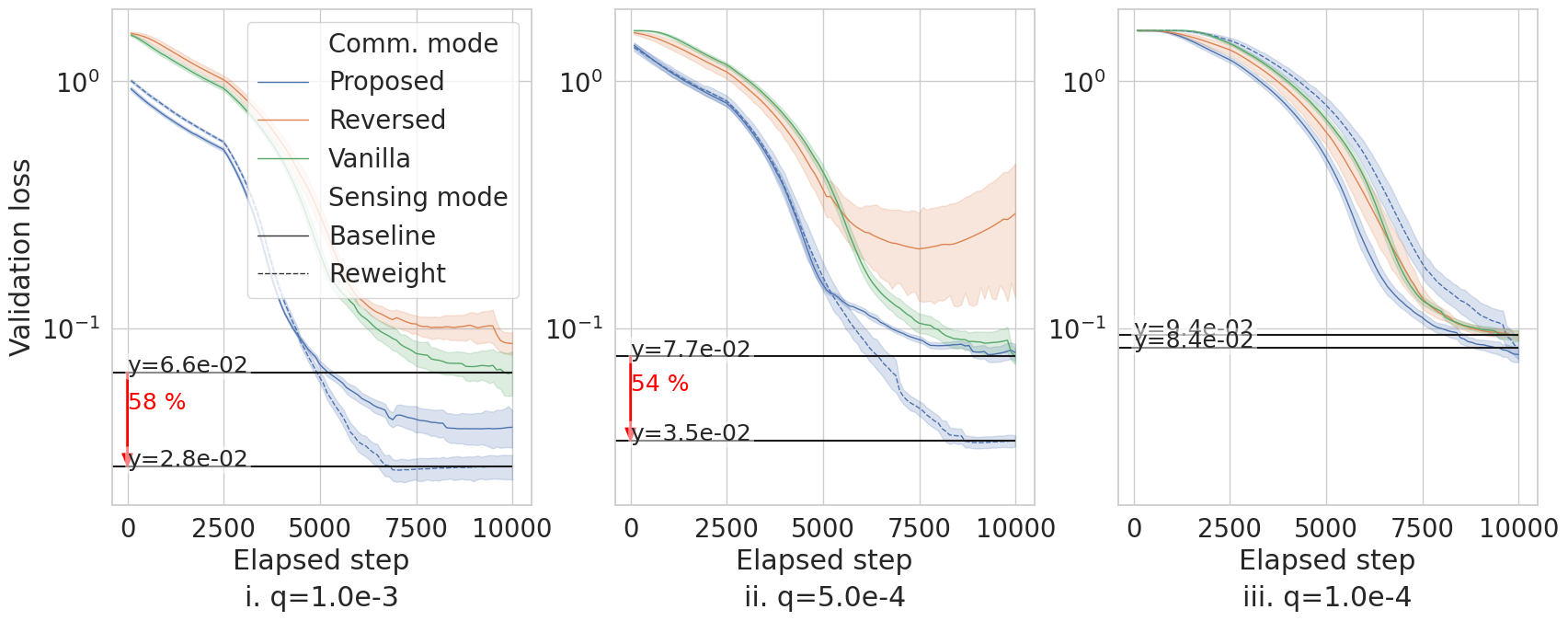}\label{fig:vloss_steps_radar}}
\hfill
\subfloat[Versus the number of used data samples.]{
		\includegraphics[width=0.48\textwidth]{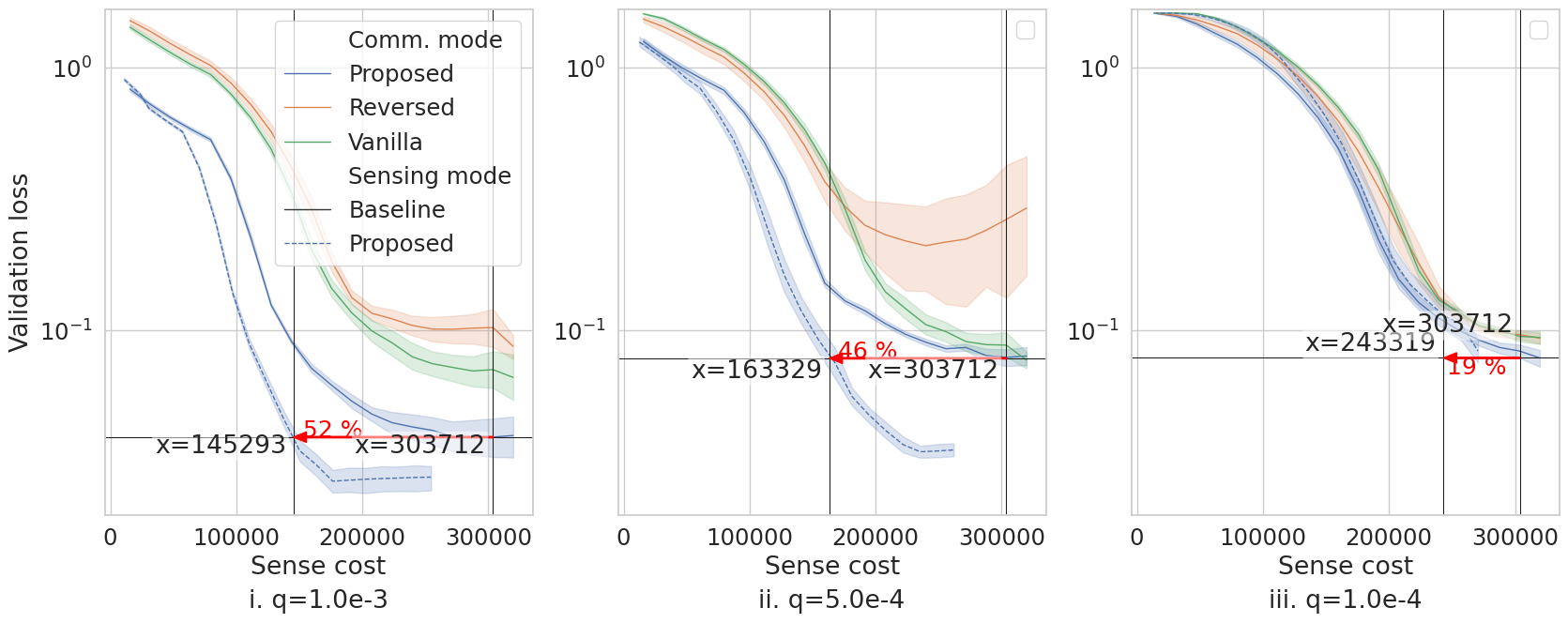}\label{fig:vloss_sensecost_radar}}
\caption{Validation curves on radar dataset.}
\label{fig:vloss_commcost}
\vspace{-10pt}
\end{figure*}

\begin{figure*}[t]
\centering
\subfloat[Radar dataset.]{
		\includegraphics[width=0.48\textwidth]{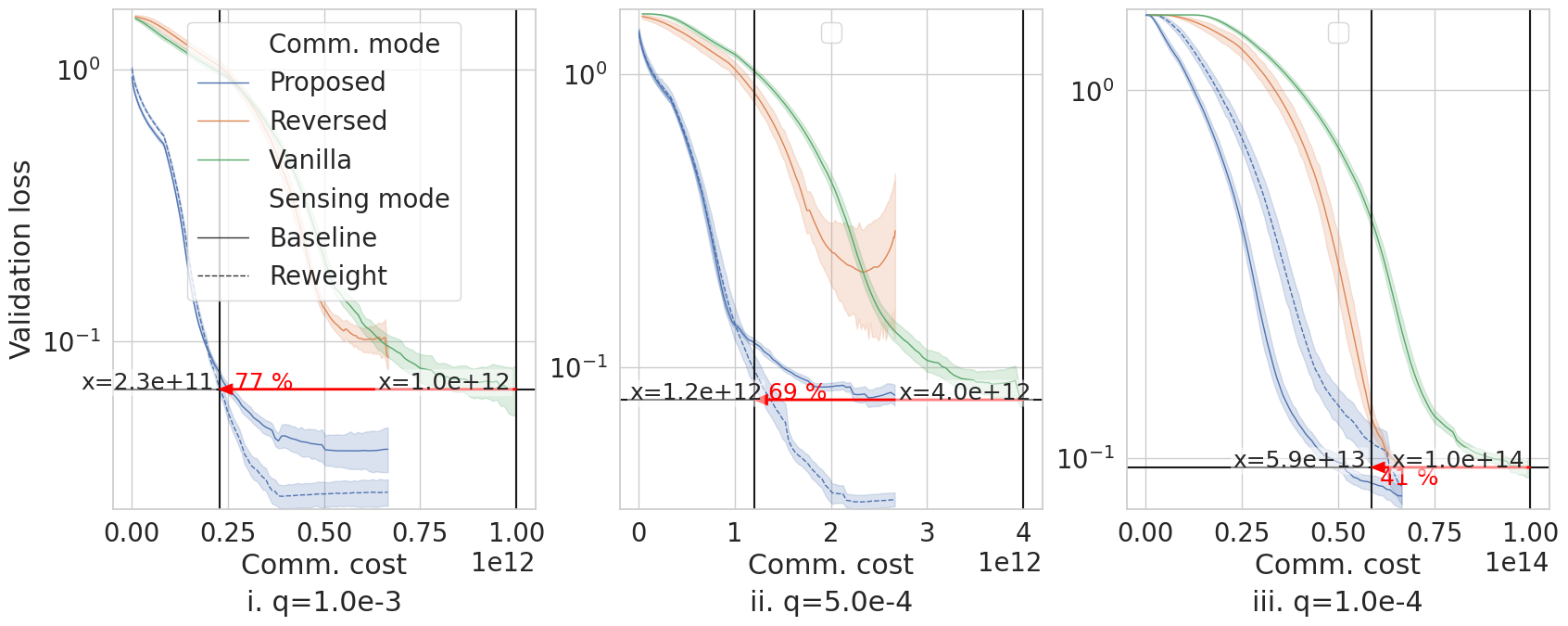}\label{fig:vloss_comm_radar}
        }
\hfill
\subfloat[MNIST dataset.]{
		\includegraphics[width=0.48\textwidth]{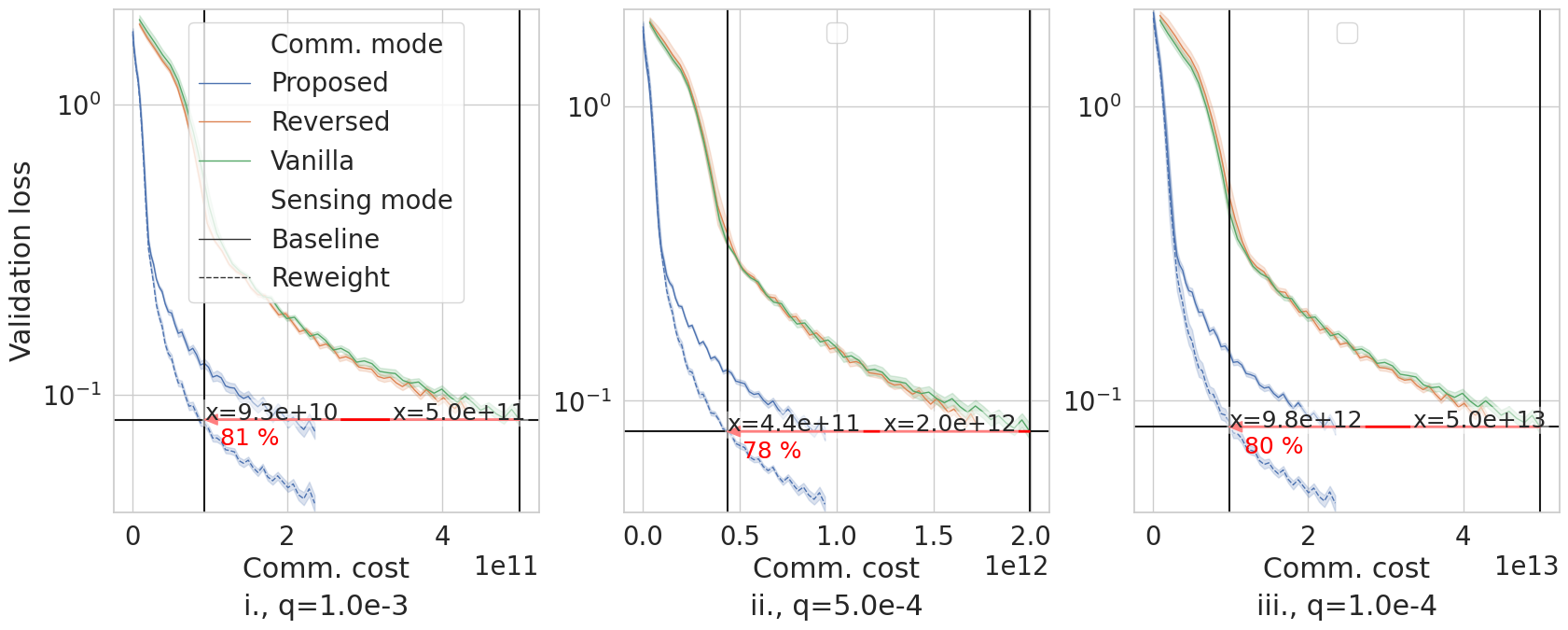}\label{fig:vloss_comm_mnist}
        }
\label{fig:vloss_commcost}
\caption{Validation loss versus communication cost.}
\end{figure*}

\begin{figure}[t]
\centering
\includegraphics[width=0.48\textwidth]{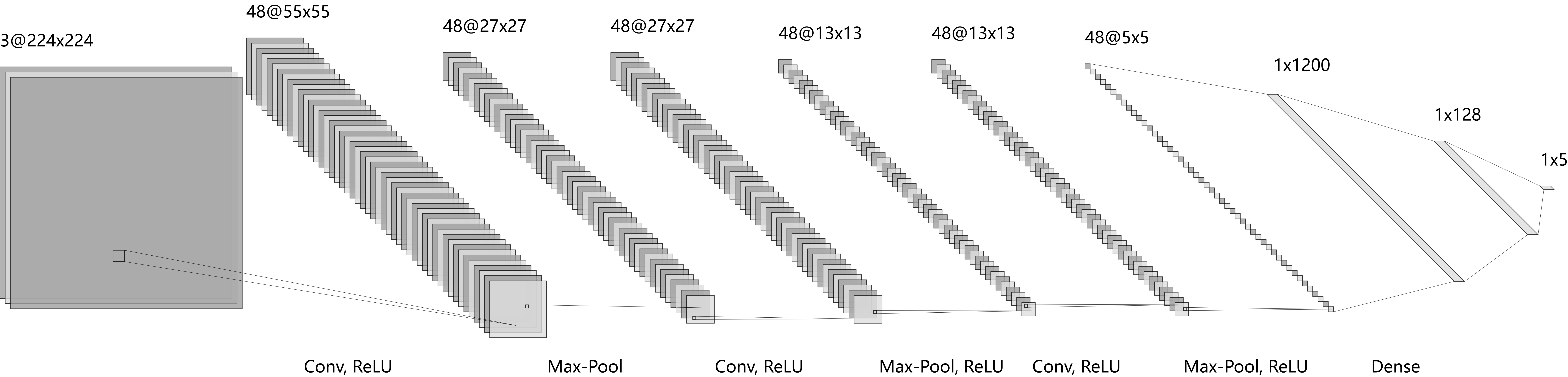}
\vspace{-10pt}
\caption{Architecture of LargeNet.}
\label{fig:largenet}
\end{figure}

In this section, we present the simulation results to validate the proposed resource allocation scheme and the sample gradient reweighting method. To fully assess the proposed scheme's performance, we use two different case studies: one uses a handwritten digits classification task, and the other uses a simulated JSAC simulation scenario with an AI-based human motion classification task.
This experiment considers an Air-FedSGD system with an edge server and $K = 5$ JSAC devices. We adopt the SGD optimizer with no momentum with learning rate $\eta = 0.01$ and fixed batch size. For each different setting, loss curves are plotted based on $5$ independent repeated experiments. For experiments considering Algorithm \ref{alg:boot}, we set $\alpha = 0.1$ and $\underline{b} = 4$.

\subsection{Proposed Methods and Baselines}

Understandably, $q$ is a parameter strongly related to the properties of $F$, so we evaluate the power control methods under different values of $q$ to prove its versatility. We first define and analyze communication energy consumption in energy units, then bridge the analysis with real-world to actual parameters. 

To verify the advantage of the JSAC resource allocation scheme, we define the following communication energy budget distributed schemes, including a proposed one and three baselines: \begin{itemize}
    \item \textbf{\textit{Proposed}}: The communication energy budget is distributed according to $c_r = p_n \sqrt{r} / \sqrt{q}$.
    \item {\it Vanilla}: The communication power is set to $c_r = p_n \sqrt{R} / \sqrt{q}$ so that the communication budget caused by environmental noise is averagely distributed to different rounds, simulating a vanilla SGLD algorithm deployed in a distributed manner.
    \item {\it Reversed}: The communication energy budget is distributed according to $c_r = p_n \sqrt{R-r} / \sqrt{q}$. This scheme is added to rule out the possibility that the performance gain is from varying levels of noise while assuring that the total communication cost is equal.
\end{itemize}

Moreover, to verify the advantage of the sample gradient reweighting technique, we define the  \textbf{\textit{Reweight}} sensing scheme that follows Algorithm \ref{alg:boot} and the \textit{Baseline} that fixes the number of data samples in each step.

We run $5$ repeats for each scheme under identical random seed and model hyperparameters.

\subsection{Evaluation Metrics}

The objective is the validation loss, measured on the test set with no intersection of the training set. The communication cost at the $r$-round is measured by the accumulated number of times of unit energy, which is defined as $u_{r'} = c_{r'} t T_1 / p_n = t T_1 \tau_{r'}^{-1}$. A higher metric implies that less overall communication noise is allowed throughout the training process. The sensing cost at the $r$-th round is measured by how many samples are used in total from the beginning of training until the $r$-th steps.

\subsection{Human Motion Classification Task}

\subsubsection{Settings}

We assume the devices, in this case, are JSAC devices with an FMCW sensor that generates the dataset. The FMCW radar operates at center frequency $60$ GHz, with a bandwidth of $10$ MHz, with variance of shadow fading $8$ dB and noise power spectral density $-174$ dBm/Hz. The chirp duration is $10 \mu$s, the sampling rate is $10$ MHz, and the number of chirps per frame is $25$. Data samples of $5$ different classes will be captured, i.e., child walking, child pacing, adult walking, adult pacing, and standing. Following \cite{liu2022toward}, we adopt a prototype-based wireless sensing simulation scheme \cite{ram2010simulation}. We utilize the open-sourced toolbox\footnote{We maintain the wireless environment setting in Sec. VI in \cite{liu2022toward}, and following the paper, using \url{https://github.com/PeixiLiu/humanMotionRadar}. Some examples can be found in the paper.} to generate spectrograms in RGB images as data samples and downsized them to $224$ by $224$ pixels.

Since the samples are larger pictures with colors, we use a larger AI model to discriminate them. Each device runs a LargeNet\footnote{Some commonly used operations like \texttt{BatchNorm}, smashes the sample-wise gradient information, hence making tracking the sample gradient during training impossible. Also, tracking sample-wise gradient is time and computation-consuming, so we only focus on AI models with simpler architectures.}, as shown in \Cref{fig:largenet}.





Fixing $B = 320000$\footnote{For clear visualization of the results, results are recorded once every $10$ step and are smoothed using a linear kernel of length $100$ with bias corrections.}. An exhaustive search on powers of $2$ between $1$ and $1000$ as batch size yields $32$ as the best choice of $b_r$ or $\bar{b}$ for the baseline sensing scheme, so the maximum batch size is set to $32$ as well in the reweight sensing scheme. 

\subsubsection{Results}
Some observations on the simulation results are as follows.


{\bf The proposed scheme yields the best generalizable models with less communication cost.} First of all, it is clear that to achieve a much lower parameter noise, the communication energy consumed will be much higher. This implies a natural approach of allowing noise of certain amplitude in the gradients without affecting learning performance much, which coincides with the findings of model quantization \cite{zhu2020one}. {\it Reversed} is designed to start at the same $c_r$ with the {\it Vanilla} scheme, then lowers gradually rather than staying constant. The accumulated consumed energy is equal to the proposed scheme but lower than that of the {\it Vanilla} scheme. It can be observed that the proposed method uses up to $77 \%$ less communication energy compared to {\it Vanilla} when reaching the same level of validation loss (\Cref{fig:vloss_comm_radar}, i.). 

{\bf Sample gradient reweighting requires less sensing cost and hits lower validation loss.} Now we take dashed lines into observation. We can observe that from Figure \ref{fig:vloss_sensecost_radar}, the dashed line 
always outperforms the other methods. In the human motion classification experiment, we can even observe that the {\it Reweight} curve is observed to converge at a lower error floor and a lot earlier than the counterparts. From \Cref{fig:vloss_sensecost_radar}, in terms of overall number of samples consumed, we can observe that the {\it Reweight} curve used up around $10-15 \%$ samples than {\it Baseline} curves and reaches up to $58 \%$ (\Cref{fig:vloss_steps_radar}, i.) lower validation loss. However, the {\it Reweight} method took up to $52 \%$ less samples (\Cref{fig:vloss_sensecost_radar}, i.) to reach equal validation loss compared to the {\it Baseline} method. It confirms that the proposed sample gradient reweighting technique can make the Air-FEEL system more sensing efficient.

\subsection{Hand Written Digits Classification Task}

\begin{figure}[t!]
\centering
\includegraphics[width=0.35\textwidth]{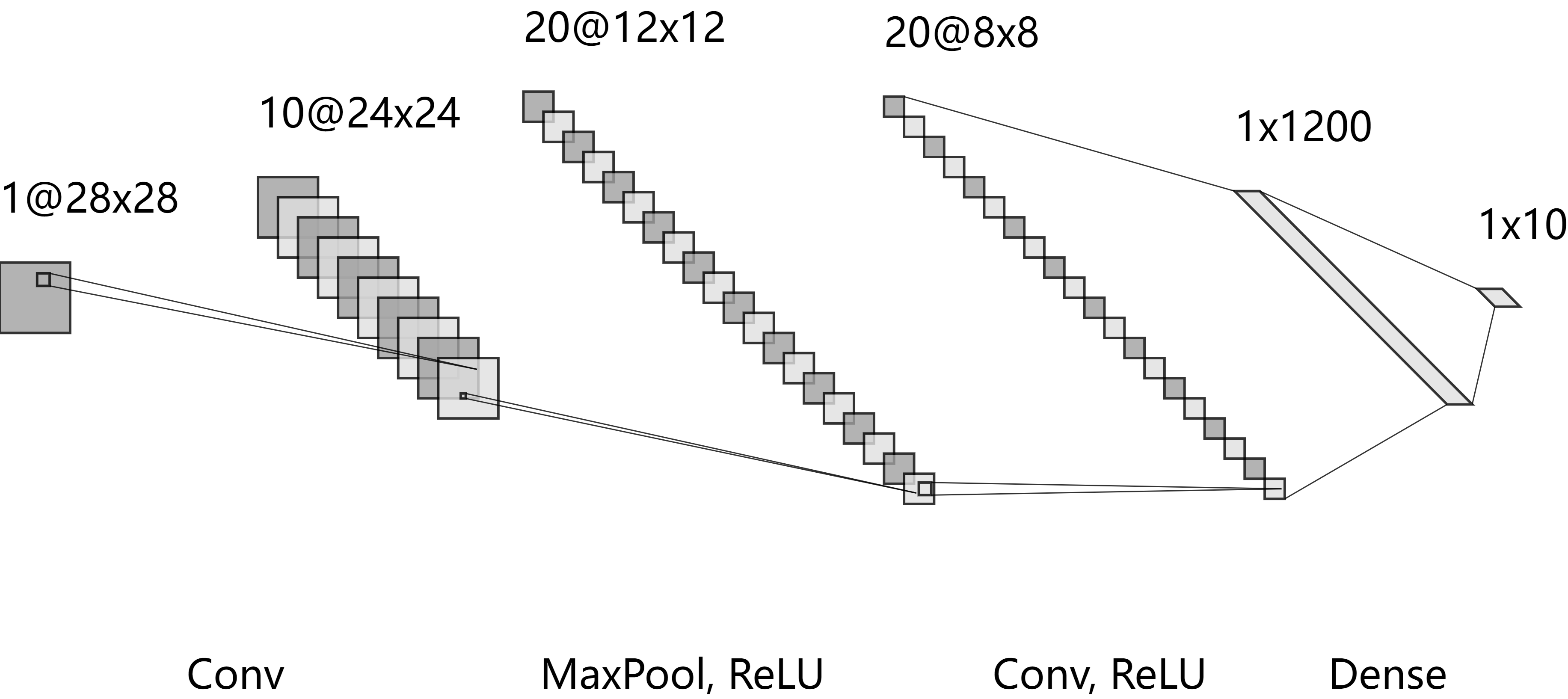}
\vspace{-10pt}
\caption{Architecture of SmallNet.}
\label{fig:smallnet}
\end{figure}



\begin{figure*}
\centering
\subfloat{
		\includegraphics[width=0.48\textwidth]{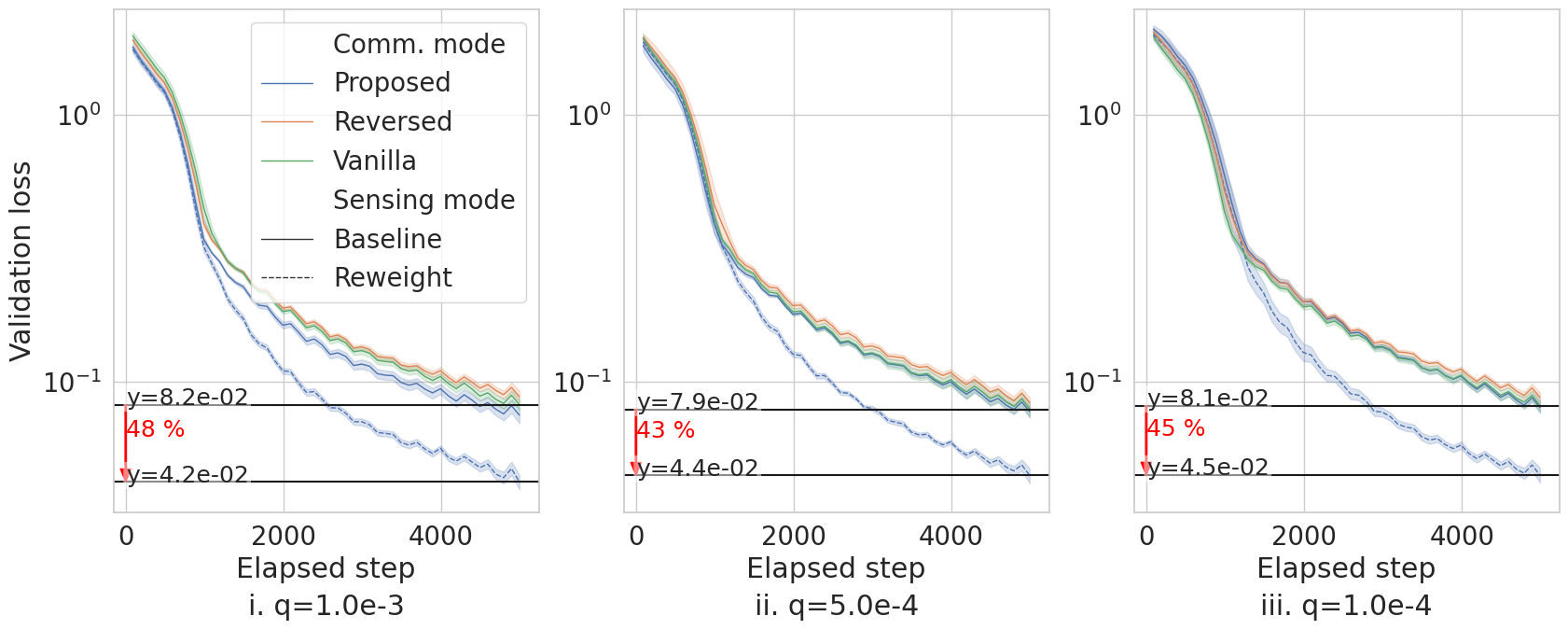}\label{fig:vloss_steps_mnist}}
\hfill
\subfloat{
		\includegraphics[width=0.48\textwidth]{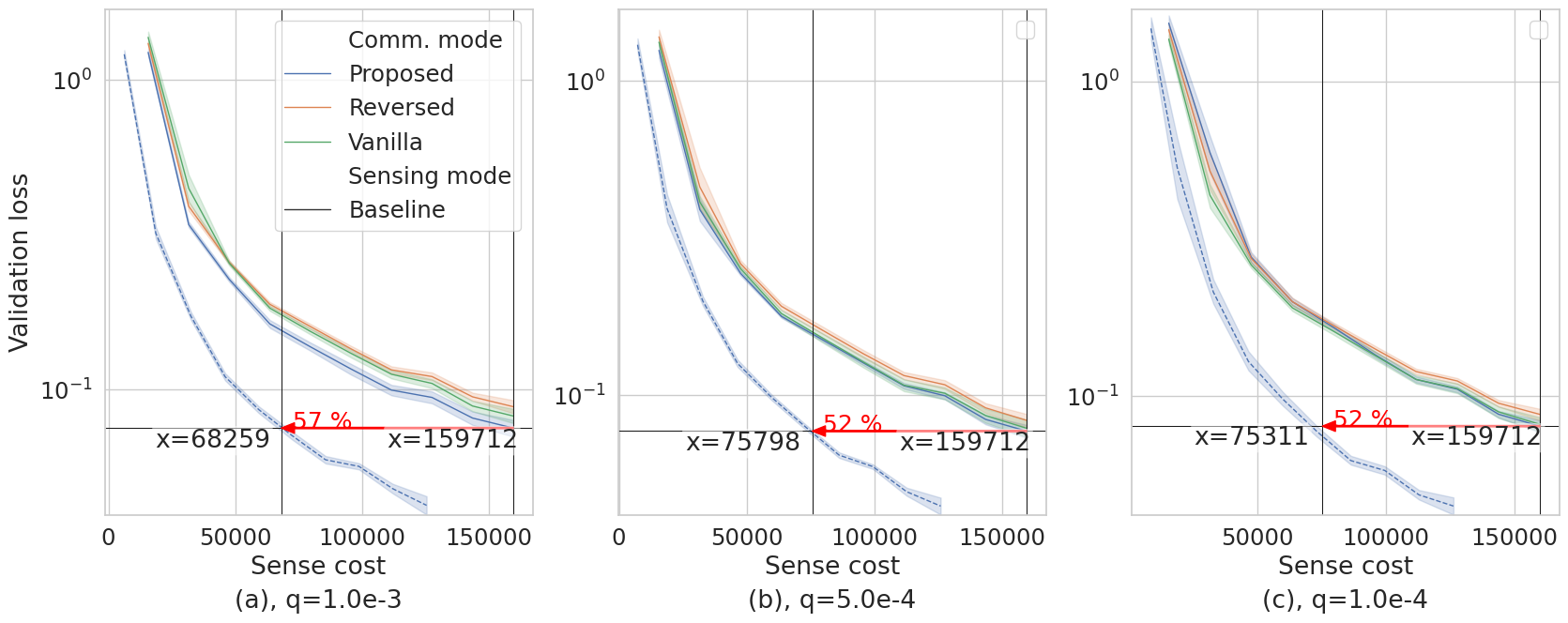}\label{fig:vloss_sensecost_mnist}}
\vspace{-10pt}
\caption{Validation curves on MNIST datasete.}
\label{fig:vloss_commcost}
\vspace{-10pt}
\end{figure*}

\subsubsection{Settings}

We consider training the FEEL system on the famous MNIST dataset \cite{lecun1998gradient} that contains around 70000 handwritten digits with values evenly distributed from 0 to 9. Each of the devices has a copy of the training set and sample data batches independently from it. We include this experiment to show that our result can be generalized to other tasks and model structures.

The devices run a SmallNet, whose structure is shown as \Cref{fig:smallnet}. We fix $B = 160000$, and keep other settings unchanged from that of the previous experiment. 

\subsubsection{Results}

The results agree with the previous observations, with an additional observation can be obtained.

{\bf Easier tasks and smaller models suffer less overfitting.} While the {\it Proposed} scheme still performs better than the other schemes, the gaps are smaller. This is because a smaller model is naturally weaker at carrying information from the dataset, resulting in smaller $I(W; D)$ and smaller overfitting consequently. This observation is aligned with the discussion in \Cref{sec:ACLGE}.

\section{Conclusion} \label{sec:conclusion}

In this paper, we investigated the issue of enhancing the generalizability of edge intelligence using AI-in-the-loop sensing and communication joint control. The solution is two-phased. Firstly, an analysis of how the data distribution and weight distribution as two different aspects impact the generalizability of edge intelligence lays down the foundation for validation loss optimization. Then the problem is split into the data distribution and weight distribution optimization subproblems. The former subproblem is met with a heuristic AI-in-the-loop sensing control algorithm, and the latter is optimized and induces a joint sensing and communication resource allocation scheme. As a byproduct of improving generalizability, the communication and sensing costs are reduced.

This work opens several future directions for further investigation. Firstly, more practical cases, such as heterogeneous JSAC devices, can be considered. Secondly, the trade-off between empirical loss and generalization error is overlooked in this paper. Thirdly, validation loss under other edge AI paradigms heavier AI models, or more complicated structures is also underexplored. Summing up, utilizing AI-in-the-loop control brings more degrees of freedom for resource allocation and joint AI-JSAC designs. Taking generalizability into consideration also opens possible synergy patterns for AI and JSAC systems, and emerges as a promising research direction.

\begin{appendices}
\crefalias{section}{appendix}

\section{Proof to \Cref{prop:norm_grad}}
\label{app:norm_grad}
\begin{proof}
  With \Cref{ass:l_continuous}, we have 
  \begin{align}
    F(\bs{w}) &= F(\bar{\bs{w}}) + \nabla F(\bar{\bs{w}}) (\bs{w} - \bar{\bs{w}}) \notag \\
    &~+ \int_0^1 \left(\nabla F(\bar{\bs{w}} + t(\bs{w} - \bar{\bs{w}})) - \nabla F(\bar{\bs{w}})\right)^T(\bs{w} - \bar{\bs{w}}) {\rm d}t \notag \\
    \leq& F(\bar{\bs{w}}) + \nabla F(\bar{\bs{w}})(\bs{w} - \bar{\bs{w}}) + L\|\bs{w} - \bar{\bs{w}}\|_2^2. \label{eq:--4}
  \end{align}

  Then it follows that  
  \begin{align}
    & \bb{E}[F(\bs{w}_r)] - F(\bs{w}_{r-1}) \notag \\
    \leq& -\eta \nabla F(\bs{w}_{r-1})^T \bb{E}[\bs{g}_r] + \frac{L \eta^2}{2} \bb{E}[\|\bs{g}_r\|_2^2] + \frac{L \tau_r \eta^2}{2K} \label{eq:-3} \\
    \leq& -\eta \mu_F \|\nabla F(\bs{w}_{r-1})\|_2^2 + \frac{L \tau_r \eta^2}{2K} + \frac{L \eta^2}{2} \left(M_e + \frac{M_v}{Kb_r}\right)\notag \\
    & + \frac{L \eta^2}{2} \left( M_E + \frac{M_V}{Kb_r} \right) \|\nabla F(\bs{w}_{r-1})\|_2^2 \label{eq:-2}\\
    \leq& -\frac{\eta \mu_F}{2} \|\nabla F(\bs{w}_{r-1})\|_2^2 + \frac{L \tau_r \eta^2}{2K} + \frac{L \eta^2}{2} \left(M_e + \frac{M_v}{Kb_r}\right), \label{eq:-1}
  \end{align}
  where \eqref{eq:-1} is made possible by taking $0 < \eta \leq \frac{\mu_F}{L\left(M_E + \frac{M_V}{Kb_r}\right)}$.
\end{proof}

\section{Proof to \Cref{lm:gen_error}}
\label{app:var_grad}
\begin{proof}
  Following the line of \cite{xu2017information, Wang2021}, we present proof for the version of the lemma that we used in this paper. We first express the generalization error with in information-theoretic term, then connect this term with the variation of gradients.

  Like \cite{russo2019much}, for the first step, we exploit the Donsker-Varadhan variational representation \cite{donsker1983asymptotic} of the relative entropy. Denote ${\rm D}(\cdot || \cdot)$ the Kullback-Liebler divergence operator, where ${\rm D}$ is in roman font. Fix two probability measures $P$ and $Q$ on a measurable space $\Omega$. Denote the set $\mc{F}$ to be the set of functions $f$ that satisfies $f : \Omega \to \bb{R}$ and $\bb{E}[\exp(f)] < \infty$. Then when ${\rm D}(P || Q) < \infty$, $\bb{E}_P[f]$ exists and \begin{align}
    {\rm D}(P || Q) = \sup_{f \in \mc{F}} \left\{\bb{E}_P[f] - \log \bb{E} [\exp(f)] \right\}.
  \end{align}

  Substitute $P$ with $P_{W, D_n}$ and $Q$ with $P_{W} P_{D_n}$, when $f$ is some measurable sample loss function, where $W$ is a weight random variable and $D_n$ a data random variable of some index $n$, and $P_{W, D_n}$ is defined by some learning algorithm. Denote $\bar{D}_n$ an identically independent copy of $D_n$, then with $P_W = P_{\bar{W}}$ and $P_{W, D_n} = P_{\bar{W}, \bar{D}_n}$, for any $\lambda \in \bb{R}$, it gives \begin{align}
    I(W; D) =& {\rm D}(P_{W, D} || P_{W} P_{D}) \notag \\
    \geq& \bb{E}[\lambda f(W, D_n)] - \log \bb{E}\left[\exp(\lambda f(W, \bar{D}_n))\right] \notag \\ 
    \geq& \lambda \left(\bb{E}[f(W, D_n)] - \bb{E}[f(W, \bar{D}_n)]\right) - \frac{\lambda^2 \sigma^2}{2} \notag \\
    \geq& 0,
  \end{align}
  which defines a nonnegative parabola in $\lambda$. The quadrative discriminant implies \begin{align}
    \left|\bb{E}[f(W, D_n)] - \bb{E}[f(W, \bar{D}_n)]\right| \leq \sqrt{2 \sigma^2 I(W; D_n)}. \notag
  \end{align} 
  By substituting the data random variable $D_n$ to dataset random variable $D$, and sample loss function $f$ to its averaged version $\hat{F}$ over every sample in the set, defined as \begin{align}
    \hat{F}(W, D) = \frac{1}{b} \sum_{i=1}^b f(W, D_i), \notag
  \end{align} it gives \begin{align}
    |\gen(W, D)| =& \left|\bb{E}[\hat{F}(W, D_n)] - \bb{E}[\hat{F}(\bar{W}, \bar{D}_n)]\right| \notag \\ 
    \leq& \sqrt{\frac{2 \sigma^2}{B} I(D, W)},
  \end{align}
  Where $B$ is the size of the dataset $D$, resulting in $\hat{F}$ being a $\sigma / \sqrt{B}$-subgaussian function.

  For the second step, denoting $\mc{D}^{(R)} = \{\mc{D}_r\}_{r \in \mc{R}}$, where $\mc{D}_{r} = \{\mc{D}_{k, r}\}_{k \in \mc{K}}$ and $\mc{D}_{k, r} = \{\bs{d}_{k, r, i}\}_{i=1}^{b_r}$, we have \begin{align}
    &I(W_r; D^{(r)}) \notag \\
    \leq& I(W^{(r)}, D^{(r-1)}; D^{(r)}) \label{eq:sdpi} \\
    =& I(W_r; D_r | W^{(r-1)}, D^{(r-1)}) + I(W^{(r-1)}, D^{(r-1)}; D_r)~\label{eq:chain_rule} \\ 
    =& \sum_{r'=1}^{r} I(W_{r'}; D_{r'} | W^{(r'-1)}, D^{(r'-1)}), \label{eq:summing}
  \end{align}
  where \eqref{eq:sdpi} is by the \emph{strong data processing inequality} and \eqref{eq:chain_rule} by the chain rule of mutual information. \eqref{eq:summing} is by repeating the chain rule, with the last residual $I(\bs{w}_0, \varnothing; \mc{D}_1)$ equals to $0$ as long as the model is trained from scratch and batches of data are taken randomly. 
  
  Since the model is considered to be continuously fed with the newest collected batches of data, we assume that almost surely that $\mc{D}_r \cap \mc{D}_{r'} = \varnothing$ for $r \neq r'$.

  Observing that \begin{align}
    &I(W_r; D_r | W^{(r-1)}, D^{(r-1)}) \label{eq:desired_term} \\
    =& I\left(\frac{\eta}{Kb_r} \sum_{k \in \mc{K}}\sum_{i=1}^{b_r}\nabla f(\bs{w}_{r-1}, \bs{d}_{k, r, i}) + \sqrt{\frac{\tau_r}{K}} \bs{n}; \mc{D}_r \right| \notag \\
    &\quad \left. W^{(r-1)} = \bs{w}^{(r-1)}, D^{(r-1)} = \mc{D}^{(r-1)} \right), \label{eq:mi_grad}
  \end{align} where $\bs{n} \sim \mc{N}(\bs{0}, \bs{I})$.

  From properties of mutual information, \begin{align}
    I(t(X) + N; X) =& {\rm D}(P_{t(X) + N} || P_N | P_X) - {\rm D}(P_{t(X) + N} || P_N) \notag \\
    \leq& {\rm D}(P_{t(X) + N | X} || P_N | P_X) \notag \\
    =& {\rm D}(P_{t(X) + N | X} || P_N) = \frac{\bb{E}[\|t(X)\|]_2^2}{2}. \notag 
  \end{align}

  Set an affine transform $t = \frac{\nabla - \bb{E}\nabla}{\sqrt{\bb{V}}}$ to make the left-most term in \eqref{eq:mi_grad} a uniform gaussian random variable and with proper arrangement, \eqref{eq:desired_term} evolves into \begin{align}
    I(W_r; D_r | W^{(r-1)}, D^{(r-1)}) = \frac{\eta^2}{2K b_r^2 \tau_r}\bb{V}[\bs{g}_r]. \label{eq:single_term}
  \end{align}

  Note that $\sum_{k \in \mc{K}}\sum_{i=1}^{b_r} f(\bs{w}_{r-1}, \bs{d}_{k, r, i})$ is a $\sigma / \sqrt{Kb_r}$-subgaussian function, plug \eqref{eq:single_term} into \eqref{eq:mi_grad}, we have \begin{align}
      \gen (W_r; D^{(r)}) \leq \frac{\sigma \eta}{K B_r} \sum_{r'=1}^r \sqrt{\frac{\bb{V}[\bs{g}_{k, r', i}]}{K \tau_{r'} b_{r'}}}. \label{eq:summation}
  \end{align}
  However, we always measure the generalization error against the whole dataset. Specifically, we regard $b^{(R)} = [b_1, \dots, b_R]$ as a predetermined series of datasets even if later batches of data are not yet collected. So the denominator $B_r$ is fixed as $B_R$, and we can easily write the increment of the generalization error as the last term of the summation in \eqref{eq:summation}.
\end{proof}

\section{Proof to \Cref{thm:J_r}}
\label{app:j_r}
\begin{proof}
    From \Cref{prop:norm_grad} we have, 
    \begin{align}
        \Delta F_r & = \bb{E}[F(\bs{w}_r)] - F(\bs{w}_{r-1}) \notag \\
        \leq & - \frac{\eta \mu_F}{2} \|\nabla F(\bs{w}_{r-1})\|_2^2 + \frac{L \tau_r \eta^2}{2 K} + \frac{L \eta^2}{2}\left(M_e + \frac{M_v}{Kb_r}\right) \notag \\
        \leq & - \eta \mu_F \delta \gamma_{r-1} + \frac{L \tau_r \eta^2}{2 K} + \frac{L \eta^2}{2}\left(M_e + \frac{M_v}{Kb_r}\right),\label{eq:Jr_part_1}
    \end{align}
    where \eqref{eq:Jr_part_1} is by \Cref{ass:plineq}.
    
    From \Cref{lm:gen_error} combined with \Cref{prop:var_svar},
    \begin{align}
        \Delta G_r = & \gen(\bs{w}_r, \mc{D}^{(R)}) - \gen(\bs{w}_{r-1}, \mc{D}^{(R)}) \notag \\
        \leq & \frac{\sigma \eta}{K B} \sqrt{\frac{1}{K \tau_r b_r} \left(\frac{2 \gamma_{r-1}}{\eta \mu_G} + \frac{L \tau_r \eta}{K \mu_G}\right)} \notag \\
        = & \frac{\sigma \eta}{K B} \sqrt{\frac{1}{K \mu_G b_r} \left(\frac{2 \gamma_{r-1}}{\eta \tau_r} + \frac{L \eta}{K}\right)} \label{ineq:Gn<An} \\
        \leq & \frac{\sigma \eta}{2 K B} \left(\frac{1}{K b_r \mu_G} + \frac{2 \gamma_{r-1}}{\eta \tau_r} + \frac{L \eta}{K}\right), \label{eq:Jr_part_2}
    \end{align}
    where \eqref{ineq:Gn<An} is from average inequality $2 \sqrt{ab} \leq a + b$.

    With \eqref{eq:Jr_part_1} and (\ref{eq:Jr_part_2}), plug in $c_r = p_n / \tau_r$, we can have \eqref{eq:expandJ} by
    \begin{align}
        \bb{E}[J_r] = \Delta F_r + \Delta G_r.
    \end{align}
\end{proof}

\end{appendices}




\ifCLASSOPTIONcaptionsoff
  \newpage
\fi



%



\bibliographystyle{IEEEtran}
\bibliography{ref.bib}

%








\end{document}